\documentclass[letterpaper]{article} 
\usepackage{aaai23}  
\usepackage{times}  
\usepackage{helvet}  
\usepackage{courier}  
\usepackage[hyphens]{url}  
\usepackage{graphicx} 
\usepackage{natbib}  
\usepackage{caption} 
\usepackage{etoolbox}
\usepackage{algorithm}
\usepackage{algpseudocode}

\usepackage{newfloat}
\usepackage{listings}
\DeclareCaptionStyle{ruled}{labelfont=normalfont,labelsep=colon,strut=off} 
\floatstyle{ruled}
\newfloat{listing}{tb}{lst}{}
\floatname{listing}{Listing}
\usepackage[utf8]{inputenc}
\usepackage{multirow}
\usepackage{booktabs,tabularx}
\usepackage{nicefrac}
\usepackage{microtype}
\usepackage{enumitem}
\usepackage[english]{babel}
\usepackage{svg}

\usepackage{caption}
\usepackage{subcaption}
\usepackage{bm}
\usepackage{amsmath, amssymb, amsthm}
\theoremstyle{definition}
\newtheorem{definition}{Definition}
\newtheorem{theorem}{Theorem}[section]

\newtheorem{proposition}[theorem]{Proposition}
\newtheorem{remark}{Remark}

\newcommand\mathloose
{%
    \thinmuskip=2mu   
    \medmuskip=2mu    
    \thickmuskip=2mu  
}
\mathloose

\title{
SimFair: A Unified Framework for Fairness-Aware Multi-Label Classification
}

\author{
    Tianci Liu,
    Haoyu Wang,
    Yaqing Wang,
    Xiaoqian Wang,
    Lu Su,
    Jing Gao\thanks{Corresponding author}
}

\affiliations{
    School of Electrical and Computer Engineering, 
    Purdue University, West Lafayette, USA, 47907

    \{liu3351,
    wang5346,
    wang5075,
    joywang,
    lusu,
    jinggao\}@purdue.edu
}

\newcommand{\beq}{\vspace{0mm}\begin{equation}}
\newcommand{\eeq}{\vspace{0mm}\end{equation}}
\newcommand{\beqs}{\vspace{0mm}\begin{eqnarray}}
\newcommand{\eeqs}{\vspace{0mm}\end{eqnarray}}
\newcommand{\barr}{\begin{array}}
\newcommand{\earr}{\end{array}}

\newcommand{\Xmat}{{\bf X}}

\newcommand{\xv}{\boldsymbol{x}}
\newcommand{\yv}{\boldsymbol{y}}

\newcommand{\E}{\mathbb{E}}

\newcommand{\ones}[0]{\boldsymbol{1}}

\newcommand{\real}{{\mathbb R}}

\newcommand*{\dif}{\mathop{}\!\mathrm{d}}

\newcommand{\fairnotion}{$s_\gamma$-SimFair}

\begin{document}

\maketitle

\begin{abstract}
Recent years have witnessed increasing concerns towards unfair decisions made by machine learning algorithms.
To improve fairness in model decisions, various fairness notions have been proposed and many fairness-aware methods are developed. 
However, most of existing definitions and methods focus only on single-label classification.
Fairness for multi-label classification, where each instance is associated with more than one labels, is still yet to establish. 
To fill this gap, we study fairness-aware multi-label classification in this paper.  
We start by extending Demographic Parity (DP) and Equalized Opportunity (EOp),
two popular fairness notions, to multi-label classification scenarios. 
Through a systematic study, 
we show that on multi-label data, because of unevenly distributed labels, 
EOp usually fails to construct a reliable estimate on labels with few instances. 
We then propose a new framework named \textbf{Sim}ilarity $s$-induced \textbf{Fair}ness ({\fairnotion}). 
This new framework
utilizes data that have similar labels when estimating fairness on a particular label group for better stability,
and can unify DP and EOp. 
Theoretical analysis and experimental results on real-world datasets together demonstrate the advantage of {\fairnotion} over existing methods 
on multi-label classification tasks. 
\end{abstract}

\section{Introduction}

Nowadays, machine learning algorithms play increasingly more important roles in decision-making for a broad spectrum of applications, 
such as applicant screening in job markets, credit risk analysis, and recommendation systems. 
However, recent studies \cite{barocas2016big, buolamwini2018gender, dressel2018accu} have discovered that machine learning algorithms tend to make discriminatory decisions. 
For example, a dataset may contain records of physicians most of whom are male. 
As a result, a job screening algorithm trained on this dataset may unfairly predict if a person is suitable for a \textit{physician} position based on their gender,
instead of education background or professional experience. 
Obviously, such favorable prediction for male applicants is \textit{unfair} to female applicants.

Formally, the algorithmic fairness issue refers to the phenomenon that machine learning algorithms make discriminatory decisions across different demographic subgroups
and give favorable predictions for some particular subgroups. Intuitively, discriminatory decisions are associated with some demographic features contained in the data, such as age, gender, and race. 
These features are referred to as \textit{sensitive} features. 
Ideally, a fair model should be able to make decisions independent of sensitive features. 
Towards this end, different fairness notions~\cite{pedreshi2008discrimination, dwork2011fairness, hardt2016equality, chouldechova2020snapshot} have been proposed.
Among them, Demographic Parity (DP) \cite{pedreshi2008discrimination} and Equalized Opportunity (EOp) \cite{hardt2016equality} are two of the most widely-used definitions. DP requires a model's decision to be independent of sensitive features, achieving a population-level fairness \cite{edwards2015censoring, madras2018learning,creager2019flexibly}. However, \citet{dwork2011fairness} showed that such population level fairness does not necessarily guarantee fairness in all label groups. 
To address this limitation, \citet{hardt2016equality} proposed to take label information into consideration and defined EOp and its stronger version Equalized Odds (EO).  Specifically, EOp requires the decision to be independent of sensitive features conditionally in the label group receiving an \textit{favorable} outcome \cite{hardt2016equality}. 
Examples of favorable outcomes include ``being admitted to a position" in job screening, and ``approval of credit card application''. 
For brevity, we refer to the label group in which each individual receives the \textit{favorable} outcome as the \textit{advantaged} group. 
With a more restrictive fairness definition, EO further requires that the decision is independent of sensitive features in each label group, 
including not only the advantaged group but also the groups receiving other outcomes.

Some methods have been proposed based on the aforementioned fairness definitions. 
However, they are focused only on scenarios where each instance is associated with a single target label \cite{hardt2016equality,woodworth2017learning,zafar2017fairness}.
In many real-world applications, multiple labels need to be predicted for an instance.
For example, in job screening, an applicant may apply for multiple positions, and the admission decision of each position is a target label of the applicant. 
Similarly, undergraduates usually submit applications to multiple programs when applying to graduate schools, 
and thus associate themselves with multiple target labels of admission.
Scenarios where each instance is associated with more than one target labels are termed as \textbf{multi-label classification} \cite{zhang2014review}. 
Obviously, fairness concerns also exist in multi-label classification scenarios. 
One straightforward approach toward fairness in a multi-label scenario is to decompose multi-label classification into multiple binary classification tasks, 
each of which judges whether a label is associated with an instance or not, 
and then apply existing fairness metric separately on each binary classification task \cite{zhang2018binary}. 
However, this naive approach ignores one unique property of multi-label classification, i.e., the correlations among labels. 
Again, take job screening as an example. 
Applicants usually apply for positions with similar requirements of skill sets and experiences at the same time, and thus the application outcomes (labels) are correlated.  
Ignoring such correlations among labels would lead to unsatisfactory classification and let alone fairness results. 
On the other hand, existing multi-label classification methods consider the correlations among labels but cannot enforce fairness in the predictions. 

Therefore, it is critical to define fairness directly in the context of multi-label classification. 
Unfortunately, we did not find existing work along this direction. This motivates us to study this problem. 
In a multi-label scenario, since different target labels usually occur together, it is more natural to treat their combinations as an \textit{advantaged} outcome (label). 
For example, the \textit{advantaged} group in the job screening example with two possible positions can be the applicants who ``received offers of position A and position B''.
Note that this definition allows us to define more general and complex advantaged groups by specifying more than one favorable labels and requires fairness on all of them.

In practice,  the discussed fairness objective is usually achieved by incorporating some fairness notions into optimization \cite{mohler2018penalized, scutari2021achieving}. Such an optimization is  non-trivial when tackling fairness issue based on this extended concept of \textit{advantaged} group in multi-label classification, where collected data is usually not evenly distributed among different labels \cite{dekel2010multi}. When few instances are in the \textit{advantaged} group (i.e., the group that has the favorable label), it may introduce unreliable fairness constraints into optimization and degrade the fairness performance.  In this work, we show that the aforementioned optimization challenges can be alleviated by utilizing information sharing among labels. Intuitively, we \textit{group} data with different but similar labels to alleviate data shortage issue, 
and then enable an EOp-like framework to incorporate fairness constraints on advantaged groups.
This will be formalized in Section \ref{sec:method}.  We refer to our framework as Similarity $s$-induced Fairness ({\fairnotion}), highlighting the crucial requirement of a similarity measure between different labels in the data grouping step.

The proposed framework {\fairnotion} is principled in the sense that it unifies DP and EOp, bringing the flexibility of leveraging a population level fairness or 
fairness on some particular label groups (e.g., the \textit{advantaged} group) per desires. The DP and EOp are two extreme cases of {\fairnotion}. When treating all labels as equally similar and ignoring their differences,  we end up with one \textit{label} group of data, in which {\fairnotion} becomes DP. 
On the contrary, if two labels are \textit{similar} only if they are the same, 
then each \textit{label} group involves only one label, in which {\fairnotion} becomes EOp. Moreover,  {\fairnotion} is able to enforce a restrictive fairness notion on the advantaged group even when data is inadequate by utilizing information from other similar label groups.



    
    
    

Our main contributions are summarized below. 
\begin{itemize}
    \item To the best of our knowledge, we are the first to investigate fairness in a multi-label classification setting.  We extend DP and EOp to the multi-label classification setting, and recognize the challenge of achieving EOp based on both theoretical and empirical studies.
    
    \item To handle the recognized challenge, we propose a novel framework, namely {\fairnotion}, to achieve the fairness objective for multi-label classification even when imbalanced label distributions exist. We further support the proposed framework with rigorous theoretical analysis.

    \item The comprehensive experiments show that the proposed framework {\fairnotion} is able to  achieve competitive and even better performance in term of DP and EOp compared to that of directly incorporating  DP and EOp into optimization respectively. 
\end{itemize}

\section{Related Work}
\label{sec:related_work}

\subsection{Algorithmic Fairness}

Most existing fairness definitions fall into two categories: group fairness~\cite{pedreshi2008discrimination, dwork2011fairness,hardt2016equality,chouldechova2020snapshot} and individual fairness~\cite{dwork2011fairness}. Group fairness requires that the probability of being assigned to a group by a model is independent of sensitive features such as gender, age and race. For example, Demographic Parity (DP) requires that the prediction is independent of sensitive features, while Equalized Odds (EO) and Equalized Opportunity (EOp) require that the prediction is conditionally independent of sensitive features in each or some label group. 
When labels are binary, this is equivalent to requiring an equality of true and false positive rates across different demographic subgroups. 
Modifications of DP and EO (EOp) have also been studied. For example, in \citet{pleiss2017fairness},  a relaxed condition is required by replacing EO with some calibration.   
Individual fairness, on the other hand, requires that a model treats similar individuals similarly \cite{dwork2011fairness}. In this work we focus on group fairness.

In order to correct the unfairness of models, many methods have also been proposed, which can be classified into one of the following three categories: pre-processing biased datasets, in-processing models during training, and post-processing the outputs of models. 
In-processing is usually the most effective way to intervene an unfair model \cite{petersen2021post}, which can be done by penalizing unfair predictions directly~\cite{mohler2018penalized, scutari2021achieving}, or by disentangling some intermediate representations (on which final predictions are made) from sensitive features~\cite{locatello2019fairness, creager2019flexibly}. 
Nevertheless,
penalty-based methods are still good and effective starting points to mitigate unfairness~\cite{mary2019fairness, kamishima2012fairness}. 

\subsection{Multi-label Classification}

Multi-label classification is a general family of classification tasks where each instance is associated with multiple target labels. This task has very broad applications~\cite{el2015experimental},  such as recommendation systems~\cite{zheng2014context, zhang2020multi}, multi-object detection~\cite{gong2019using, zhao2020adaptive}, and text classification~\cite{yang2009effective, nam2014large}. 

Methods for multi-label classification can be grouped into two categories~\cite{zhang2014review, tsoumakas2006review}: problem transformation and algorithm adaptation. Problem transformation tackles multi-label classification by transforming the task into other well-defined tasks. One possible transformation is binary relevance~\cite{boutell2004learning}, which ignores all dependencies among different targets and predicts each target separately. Classifier chain, the other extreme case, learns the joint distribution of different labels by applying the chain rule of probability~\cite{read2011classifier}.
In summary, these problem transformation multi-label classification tasks into other well-established learning problems and adopt existing methods to solve them~\cite{tsoumakas2007random, furnkranz2008multilabel}.
Algorithm adaptation, on the other hand, modify existing algorithms such as kNN~\cite{zhang2007ml} and decision tree~\cite{clare2001knowledge} to model multi-label data directly. We refer readers to \citet{zhang2014review, tsoumakas2006review} for more details. 

Deep learning has advanced multi-label classification as well~\cite{liu2021emerge}. Recently, \citet{chen2018end,bai2020mpvae} revisited the Multivariate Probit (MP) model~\cite{chib1998analysis} with the equipment of deep learning tools. MP model assumes that the joint distribution of labels is controlled by a multivariate Gaussian random variable, and the probability of a label is determined by the cumulative density function (CDF) at the value of this Gaussian variable. The correlations in the Gaussian variable allows the model to capture pairwise dependencies in a multi-label setting.
~\citet{chen2018end} parameterized the MP model with a deep neural network resulting in the deep Multivariate Probit model (DMVP), and~\citet{bai2020mpvae} proposed to combine DMVP and variational autoencoder~\cite{kingma2014autoencoding} to obtain better performance.


\section{Methodology}
\label{sec:method}



In this section, we propose {\fairnotion}, 
a flexible framework to unify Demographic Parity (DP) and Equalized Opportunity (EOp).
We start with deriving DP and EOp in multi-label scenarios. 
Then we provide a systematic study on the challenges of estimating EOp in multi-label scenarios. We propose   {\fairnotion} based on the these studies to achieve the fairness objective even when imbalanced label distributions exist. 

\subsection{Preliminaries}
\paragraph{Notations}
Throughout this paper, we use bold capital letters (e.g., $\Xmat$) to denote matrices, 
bold lowercase letters (e.g., $\xv$) to denote (column) vectors, and calligraphic letters (e.g., $\mathcal X$) to denote spaces.
Finally, capital $P$ denotes a probability and lowercase $p$ denotes a distribution. 
We summarize notations used in this paper in appendix \ref{app:notation} for better readability.

Consider a dataset that contains $N$ samples $\mathcal D = \{ (\xv^{(i)}, a^{(i)}, \yv^{(i)}) \}_{i=1}^N$. 
Without loss of generality, 
we assume each sample is associated with $M$ non-sensitive features $\xv^{(i)} \in \mathcal X = \real^M$, 
a $K$-way scalar sensitive feature $a^{(i)} \in \mathcal A = \{1, \dots, K\}$ 
where $K$ is the number of demographic subgroups 
(e.g., $K = 2$ if \textit{gender} is the sensitive feature that takes \textit{female} and \textit{male}), 
and $L$ binary labels $\yv^{(i)}\in \mathcal Y = \{0, 1\}^L$. 
We further assume $N$ samples are drawn from an unknown underlying distribution $p$ over space $(\mathcal X \times \mathcal A \times \mathcal Y)$,
and use $(\xv, a, \yv) \sim p$ to denote a random sample. 
To avoid ambiguity, for $\yv = (y_1, \dots, y_L)$, 
we call $\yv$ a \textit{label}, 
and $y_l \in \{ 0, 1\}$ the $l$-th \textit{target},
where $y_l=1$ indicates the presence of $l$-th target.   
We use $h: \mathcal X \rightarrow \mathcal Y$ to denote a multi-label classifier that predicts label based on non-sensitive features. 
Under these settings, $L=1$ corresponds to single-label classification,
and $L > 1$ corresponds to multi-label classification.

\paragraph{Multi-Label Classification Prediction} 

We consider a wide family of multi-label classifiers that satisfy $h = f \circ g: \mathcal X \rightarrow [0, 1]^L \rightarrow \mathcal Y$. 
In particular, a classifier first predicts $\tilde \yv = g(\xv)$, the probability of the presence of $L$ targets given $\xv$.  
Then $l$-th target prediction is given by $\hat y_l = \ones(\tilde y_l \geq 0.5)$ elementwisely, 
in which $f$ denotes this elementwise thresholding function. 
This family of classifiers is capable of capturing dependencies between different targets by predicting $\tilde \yv$ given $\xv$ jointly as shown in
\citet{chen2018end,bai2020mpvae}. 



\subsection{DP and EOp on Multi-Label Classification}
\label{sec:dp-eop}

\paragraph{DP and EOp Condition}
In this section, we establish DP and EOp condition in multi-label scenarios. 
For classifier $h = f \circ g: \mathcal X \rightarrow \mathcal Y$ and random sample $(\xv, a, \yv) \sim p$, 
$h$ is \textit{fair} in terms of
(1) DP if  $\hat \yv \perp a$; and (2) EOp if $\hat \yv \perp a \mid \yv_{adv}$, where $\yv_{adv} \in \mathcal Y$ denotes some advantaged label where only favorable outcomes (e.g., ``received offer'' in the job screening example) present.
In essence, DP requires predictions to be independent with sensitive variables, and EOp requires conditional independence to hold on label $\yv_{adv}$. 
As assumed,  
prediction $\hat \yv$ depends on predicted probability $\tilde{\yv}$ elmentwisely,
therefore distribution of $\hat \yv$ is fully parameterized by $\tilde{\yv}$. 
Proposition \ref{thm:dp-eop-mlc} gives a condition for DP and EOp to hold in multi-label classification.

\begin{proposition}[DP and EOp condition for multi-label classifier]
\label{thm:dp-eop-mlc}
For a multi-label classifier that takes the form $h = f \circ g$, 
where $\tilde \yv = g(\xv)$ is the predicted probability and $\hat \yv = f(\tilde \yv)$ is computed elementwisely, 
DP and EOp hold if for any $ k \in \mathcal A$
\begin{align}
    \text{DP: } 
    &\quad 
    \E [ \tilde \yv \mid a=k ] = \E [\tilde \yv ] \notag \\
    \text{EOp: }
    &\quad
    \E [ \tilde \yv \mid a=k, \yv=\yv_{adv}] = \E [\tilde \yv \mid \yv=\yv_{adv}]. \label{eq:eop}
\end{align}
\end{proposition}

\begin{proof}
See appendix \ref{app:dp-eop-mlc}.
\end{proof}

\begin{remark}
Proposition \ref{thm:dp-eop-mlc} indicates that on multi-label data where labels are correlated,
for classifier $h$, we can still evaluate its fairness performances in the same way as evaluating traditional single label classifiers
by comparing the averaged predicted probability on different subgroups. 
Moreover, we can construct estimations with finite samples
\end{remark}
{\footnotesize
\begin{align}
    &\E [ \tilde \yv \mid a=k ] 
    \approx 
    \frac{\sum_{i=1}^N \tilde{\yv}^{(i)} \ones(a^{(i)} = k)}{\sum_{i=1}^N \ones(a^{(i)} = k)}
    \quad 
    \E [\tilde \yv ]
    \approx 
    \frac{1}{N}
    \sum_{i=1}^N \tilde{\yv}^{(i)} \notag \\
    &\E [ \tilde \yv \mid a=k, \yv=\yv_{adv}] 
    \approx
    \frac{\sum_{i=1}^N \tilde{\yv}^{(i)} \ones(a^{(i)} = k) \ones(\yv=\yv_{adv})}{\sum_{i=1}^N \ones(a^{(i)} = k) \ones(\yv=\yv_{adv})}
    \label{eq:eop-sub} \\ 
    &\E [\tilde \yv \mid \yv=\yv_{adv}]
    \approx
    \frac{\sum_{i=1}^N \tilde{\yv}^{(i)} \ones(\yv=\yv_{adv})}{\sum_{i=1}^N \ones(\yv=\yv_{adv})}
    \label{eq:eop-full}
\end{align}
}





\paragraph{Estimation Challenge of EOp} In multi-label scenarios, a long-tailed phenomenon, i.e., most labels only associate with few samples \cite{dekel2010multi}, brings additional challenges for EOp estimation.  Without sufficient samples, EOp is barely able to  construct reliable estimates for fairness and correspondingly may not achieve fairness objective in an in-processing framework on such datasets.


Mathematically, the challenge of estimating EOp stems from terms $\sum_{i=1}^N \ones(a^{(i)}=k)\ones(\yv^{(i)} = \yv_{adv})$ and $\sum_{i=1}^N \ones(\yv^{(i)}=\yv_{adv})$ in eqn~\eqref{eq:eop-sub} and \eqref{eq:eop-full}. 
When these summations are close to 0, the two estimates are unstable or even undefined. 
More formally, the conditional expectation in eqn~\eqref{eq:eop} is 
\begin{align*}
    \E [\tilde{ \yv} \mid a = k, \yv = \yv_{adv}]
    &= 
    \int \tilde{ \yv} p(\tilde{ \yv} \mid a = k, \yv = \yv_{adv}) \dif \tilde{\yv} \\
    &= 
    \frac{\int \tilde{ \yv} p(\tilde{ \yv},  a = k, \yv = \yv_{adv}) \dif \tilde{\yv}} {P(a = k, \yv = \yv_{adv})}.
\end{align*}

Here $P(a = k, \yv = \yv_{adv}) = \E [\ones(a=k) \ones(\yv = \yv_{adv})]$, and
\begin{align*}
    & 
    \int \tilde{ \yv} p(\tilde{ \yv},  a = k, \yv = \yv_{adv}) \dif \tilde{ \yv} \\
    &= 
    \iiint \tilde{ \yv} \ones(a=k) \ones( \yv = \yv_{adv}) p(\tilde{ \yv},  a, \yv) \dif a \dif \yv \dif \tilde{ \yv} \\
    &= 
    \E [\tilde{ \yv}  \ones(a=k) \ones(\yv = \yv_{adv})].
\end{align*}
This implies 
{\footnotesize
\begin{align}
    \E [\tilde{ \yv} \mid a = k, \yv = \yv_{adv}] 
    &= 
    \frac{\E [\tilde{ \yv} \ones(a=k) \ones( \yv = \yv_{adv})]}{\E [\ones(a=k) \ones(\yv = \yv_{adv})]} \\
    \E [\tilde{ \yv} \mid \yv = \yv_{adv}]
    &= 
    \frac{\E [\tilde{ \yv} \ones( \yv = \yv_{adv})]}{\E [\ones( \yv = \yv_{adv})]}.
\end{align}
}
Henceforth, eqn~\eqref{eq:eop} is equivalent to 
{\footnotesize
\begin{align}
  \frac{\E [\tilde{ \yv} \ones(\yv = \yv_{adv})]}{\E [\ones(\yv = \yv_{adv})]}
  = \frac{\E [\tilde{ \yv}  \ones(a=k) \ones(\yv = \yv_{adv})]}{\E [\ones(a=k) \ones( \yv = \yv_{adv})]}  \label{eq:eo-2}
\end{align}
}
for $\forall\ k \in \mathcal A$.
If event $\ones( \yv = \yv_{adv}) = 1$ happens with low probability, 
i.e., few samples are from label group $\yv_{adv}$, 
EOp is difficult and even impossible to estimate from eqn~\eqref{eq:eop-sub} and \eqref{eq:eop-full} directly. 



\subsection{Similarity $s$-induced Fairness ({\fairnotion)}}
\label{sec:fairnotion}

Motivated by the above analysis, 
we propose a new framework to help achieve DP or EOp, 
where hard $\ones(\yv = \yv_{adv}) \in \{ 0, 1 \}$ is relaxed to some similarity function $s(\yv, \yv_{adv}) \in [0, 1]$.
Informally, we loosen the membership of the advantaged group requirement in EOp and use a soft conditioning.
For any random sample $(\xv, a, \yv)$, fairness of its prediction is always taken in consideration, 
but as the affinity of $\yv$ to $\yv_{adv}$ decreases, 
it will be down-weighted when estimating fairness violations with respect to $\yv_{adv}$.

\begin{definition}[{\fairnotion}]
\label{def:fairnotion}
Given a similarity function $s: \mathcal Y \times \mathcal Y \rightarrow [0, 1]$, 
a multi-label classifier $h$ satisfies \textbf{Sim}ilarity $s$-induced \textbf{Fair}ness ({\fairnotion}) if for $\forall\ k \in \mathcal A$,
\begin{align}
    \frac{\E [\tilde{ \yv} s( \yv,  \yv_{adv})]}{\E [s( \yv,  \yv_{adv})]}
    = \frac{\E [\tilde{ \yv}  \ones(a=k) s( \yv,  \yv_{adv})]}{\E [\ones(a=k) s( \yv,  \yv_{adv})]}. \label{eq:fairnotion}
\end{align}
\end{definition}

Same as DP and EOp, terms involved in eqn~\eqref{eq:fairnotion} can be estimated with 
\begin{align}
    \resizebox{0.58\linewidth}{!}{$\displaystyle
    \frac{\E [\tilde{ \yv} s( \yv,  \yv_{adv})]}{\E [s( \yv,  \yv_{adv})]}
    \approx 
    \frac{ \sum_i \tilde{\yv}^{(i)} s( \yv^{(i)},  \yv_{adv}) }{ \mathbb \sum_i s( \yv^{(i)},  \yv_{adv}) }$} \hspace*{3.7em}\label{eq:fairnotion-est-all} \\
    \resizebox{0.85\linewidth}{!}{$\displaystyle
    \frac{\E [\tilde{ \yv}  \ones(a=k) s( \yv,  \yv_{adv})]}{\E [\ones(a=k) s( \yv,  \yv_{adv})]} 
    \approx 
    \frac{\sum_i \tilde{\yv}^{(i)} \ones(a^{(i)}=k) s( \yv^{(i)},  \yv_{adv}) }{ \sum_i \ones(a^{(i)}=k) s( \yv^{(i)},  \yv_{adv}) }.$} \label{eq:fairnotion-est-sub}
\end{align}

In this paper, 
we adopt the Jaccard score to define similarity $s$. 
In essence, we use the cardinality ratio between the intersection and union of pair $(\yv, \yv')$ to measure their similarity,
then apply some monotonic transformation for scaling. 
Formally, 
for $\yv \in \mathcal Y$ with $y_l = 1$ represents the presence of the $l$-th target, 
we denote $ \text{cate}(\yv) = \{ l: y_l = 1 , l = 1, \dots, L\}$, 
i.e., the collection of indices of present targets,
and define 
\begin{align*}
    \text{Jac}(\yv, \yv_{adv}) 
    &= \frac{|\text{cate}(\yv) \cap \text{cate}(\yv_{adv}) |}{|\text{cate}(\yv) \cup \text{cate}(\yv_{adv}) |} \\ 
    s_\gamma(\yv, \yv_{adv}) 
    &= \exp\left(\gamma \left( \text{Jac}(\yv, \yv_{adv})   - 1 \right) \right)
\end{align*}
where $\gamma$ is a scaling parameter. 
It is worth mentioning that the choice of $s$ is not unique and can be task- or data-specific.

\subsection{{\fairnotion} Unifies DP and EOp}
One key characteristic of {\fairnotion} is that it can be seen as an unification of DP and EOp, 
as formalized by Proposition \ref{thm:fairnotion-special} and \ref{thm:fairnotion-unify}.

\begin{proposition}[DP and EOp are special cases of {\fairnotion}]
\label{thm:fairnotion-special}

Consider {\fairnotion} defined in eqn~\eqref{eq:fairnotion}, 
if similarity $s$ is a constant function $s(\yv, \yv')  = c$ for some $c$, 
then {\fairnotion} implies DP;
if $s$ is an indicator function $s(\yv, \yv') = \ones(\yv = \yv')$, 
then {\fairnotion} implies EOp.

\end{proposition}

\begin{proof}
    See appendix \ref{app:fairnotion-special}.
\end{proof}
\begin{proposition}[{\fairnotion} helps achieve DP and EOp]
\label{thm:fairnotion-unify}
For any multi-label classifier $h$ satisfying {\fairnotion}, 
its violation of DP will be arbitrarily small if $\gamma$ is sufficiently small; 
and its violation of EOp will be arbitrarily small if $\gamma$ is sufficiently large. 
More generally, 
its violation of DP is arbitrarily close to its violation of {\fairnotion} for sufficiently small $\gamma$, 
and its violation of EOp is arbitrarily close to its violation of {\fairnotion} for sufficiently large $\gamma$. 
\end{proposition}
\begin{proof}
    See appendix \ref{app:fairnotion-unify}.
\end{proof}

\begin{remark}
Proposition \ref{thm:fairnotion-special} reveals the connection between {\fairnotion} and DP (EOp).
Proposition \ref{thm:fairnotion-unify} further shows 
that {\fairnotion} condition indeed helps achieve DP and EOp, establishing a theoretical foundation of \textit{borrowing information from similar labels}. 

\end{remark}



\subsection{{\fairnotion} Regularized Model Training}
\label{sec:method-fair-train}

\paragraph{Fairness Violation} 
Violation of {\fairnotion} 
denoted by 
$\ell_{s_\gamma (\yv, \yv_{adv})} (h)$, 
is defined as
{\footnotesize
\begin{align}
	\sum_{k=1}^K 
	\left \| 
	\frac{\E [\tilde{ \yv} s_\gamma( \yv, \yv_{adv})]}{\E [s_\gamma( \yv, \yv_{adv})]} - 
	\frac{\E [\tilde{ \yv}  \ones(a=k) s_\gamma( \yv, \yv_{adv})]}{\E [\ones(a=k) s_\gamma ( \yv, \yv_{adv})]} 
	\right \| \label{eq:fair-violation-k}
\end{align}
}
where $\| \cdot \|$ is the $L_2$ norm. 
In words, we count how the fairness conditions in eqn \eqref{eq:fairnotion} are violated in all demographic subgroup $a=k$.
When $K=2$ (i.e., the sensitive feature is binary),
it can also be writen as
{\footnotesize
\begin{align}
    \left \| 
    \frac{\E [\tilde{ \yv}  \ones(a=1) s_\gamma( \yv, \yv_{adv})]}{\E [\ones(a=1) s_\gamma ( \yv, \yv_{adv})]}
    - 
    \frac{\E [\tilde{ \yv}  \ones(a=2) s_\gamma( \yv, \yv_{adv})]}{\E [\ones(a=2) s_\gamma ( \yv, \yv_{adv})]} 
    \right \|. \label{eq:fair-violation-k2}
\end{align}
}
DP and EOp, as discussed, are special cases of {\fairnotion} so we omit their forms.

\paragraph{In-processing with {\fairnotion}}
We use {\fairnotion} to improve fairness of classifier $h$ in an in-processing framework.
Specifically, on each mini-batch during training, 
we estimate the fairness violation (defined in eqn~\eqref{eq:fair-violation-k} or \eqref{eq:fair-violation-k2}) by eqn~\eqref{eq:fairnotion-est-all} and \eqref{eq:fairnotion-est-sub}.
The estimate defines the regularization term as parts of training loss. 
In particular, we train $h$ with stochastic gradient descent-based methods by minimizing 
\begin{align}
	\min_h \ell_{\text{mlc}}(h) + \lambda \ell_{s_\gamma (\yv, \yv_{adv})}(h). \label{eq:final-loss}
\end{align}
Here $\ell_{\text{mlc}}(h)$ is the loss for multi-label classification, and
$\ell_{s_\gamma (\yv, \yv_{adv})}(h)$ estimates the violation of {\fairnotion}. 
Hyperparameter $\lambda \geq 0$ balances the two losses.

\paragraph{Multivariate Probit Variational AutoEncoder (MPVAE)}
We use MPVAE as a backbone model to illustrate and verify the performance of our work. 
MPVAE is a multi-label classification method without fairness constraint enforcement, and we adapt it with a fairness penalty to ensure {\fairnotion}. 
\textcolor{red}{}

MPVAE is a variational autoencoder structured model that is capable to capture pairwise dependency in label $\yv$. 
It learns two encoders to map $\xv$ and $\yv$ into a shared representation space and decode with the same decoder. 
A Multivariate Probit (MP) model is used to predict $\hat \yv$, 
and model correlations between different labels $y_d$ and $y_d'$. 
Fig. \ref{fig:train-pipeline} illustrates the structure of MPVAE, 
where green color marks the additional fairness penalty.
Algorithm \ref{algo:fairnotion} in appendix \ref{app:algorithm} provides a concise summary of updating MPVAE with one step on a minibatch. 
Due to the page limitation, we refer readers to \citet{bai2020mpvae} for details about MPVAE.

\begin{figure}[tb]
    \centering
    \includegraphics[width=\linewidth]{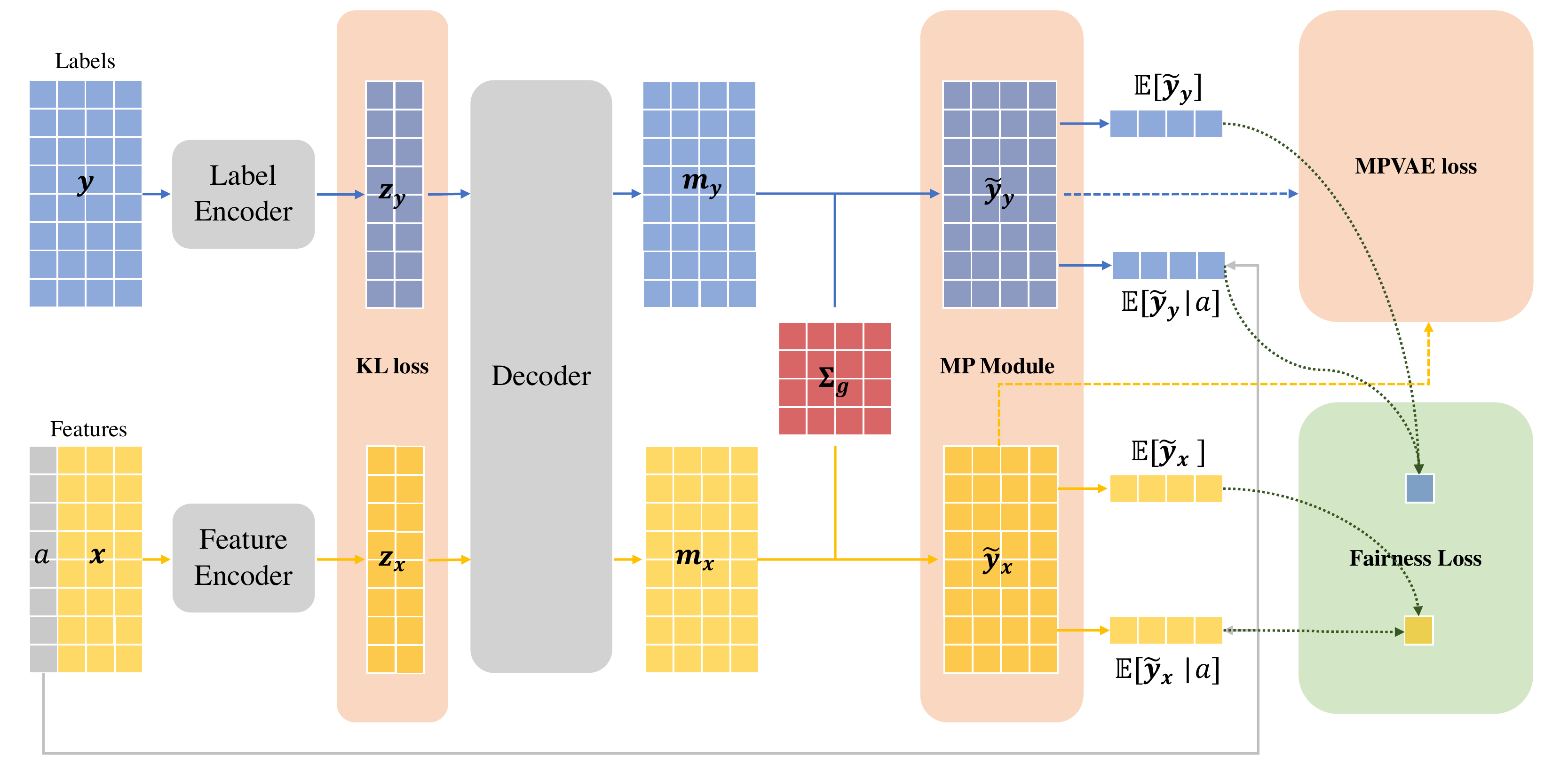}
    \caption{
    Framework of training MPVAE \cite{bai2020mpvae} with fairness regularization (in green). 
    Blocks in blue mark the \textit{label} branch and blocks in yellow mark the (non-sensitive) \textit{feature} branch. 
    During training, MPVAE predicts two probability vectors $\tilde{\yv}$ on two branches separately. 
    Both of them are used to construct the {\fairnotion} regularizer. 
    During testing, only yellow blocks (prediction from the \textit{feature} branch) are accessible.
    }
     \label{fig:train-pipeline}
\end{figure}


\section{Experiments}
\label{sec:experiment}

In this section, we evaluate {\fairnotion}
with the goal of providing insights from three aspects: 


\begin{itemize}
    \item How does {\fairnotion} approximate DP and EOp? 
    \item How does {\fairnotion} help achieve DP and EOp? 
    \item How does {\fairnotion} affect fairness-accuracy tradeoff? 
\end{itemize}

In the following, we will discuss experiment settings first and then present the details about the evaluation from these three aspects. 
\subsection{Datasets and Experiment Setup}

\subsubsection{Datasets}
\label{sec:data-mlc}

Due to the lack of existing work in fairness-aware multi-label classification,
we transform two tabular datasets that are ubiquitous in fairness literature into multi-label settings. 
Towards this goal, we select some features and treat them as additional targets. 
To help focus on the challenge brought by multi-label, we use binary sensitive features,
but as defined in eqn~\eqref{eq:eop-full}, our methods can easily generalize to where more complicated sensitive features are used\footnote{
See appendix~\ref{app:multi-class} for experiments where a multi-class sensitive feature \textit{race} is considered.}.

\begin{itemize}[leftmargin=0em]
\item \textit{Adult} \cite{kohavi1996scale} is a widely-used fairness dataset from UCI repository that contains 48,842 samples. 
Original Adult dataset contains 112 features and a binary label \textit{income level}, which denotes whether an individual's yearly income is greater than \$50K dollars or not. We further use \textit{workclass} and \textit{occupation} as two other targets.
In terms of sensitive features, 
we follow \citet{reddy2021benchmarking} and binarize \textit{age} into \textit{25-44 years old} and \textit{else}.
This allows us to construct two balanced demographic subgroups.

\item \textit{Credit} \cite{yeh2009comparisons} is another popular fairness dataset from UCI repository.
It contains 30,000 samples,
each sample is associated with 24 features 
and a binary label indicates the existence of \textit{default payments}.
We treat \textit{education level} as an additional target,
and use \textit{gender} as the sensitive feature.
\end{itemize}

\subsubsection{Baselines}
We compare MPVAE $h$ trained with proposed {\fairnotion} regularizer with three baseline methods: 
(1) No regularizer: use $\ell_{mlc}$ loss only by setting $\lambda = 0$ in eqn~\eqref{eq:final-loss};
(2) DP regularizer: use DP violation as a regularizer, can be seen as an extension of \citet{calders2009building};
and (3) EOp regularizer: use EOp violation as a regularizer, which can be seen as an extension of \citet{zafar2017fairness}. 
Regularizers are constructed according to eqn~\eqref{eq:fair-violation-k2}.
    
    
%

\subsubsection{Evaluation Metrics}
To evaluate the fairness mitigation, we report the values of eqn~\eqref{eq:fair-violation-k2} on test sets.
As these are violations of fairness, smaller values indicate better performance. 
To evaluate the multi-label classification, 
we report three popular metrics in multi-label classifications \cite{wu2017unified, bai2020mpvae}: 
micro-averaged F1 (micro-F1), macro-averaged F1 (macro-F1), and example-averaged F1 (example-F1) as defined below
{\footnotesize
\begin{align*}
    \text{micro-F1}
    &= 
    \frac{2 \sum_{l=1}^L \sum_{i=1}^N y^{(i)}_{l} \hat y_{l}^{(i)} }{\sum_{l=1}^L \sum_{i=1}^N  (y^{(i)}_{l} + \hat y_{l}^{(i)}) } \\
    \text{macro-F1}
    &= 
    \frac{1}{L} \sum_{l=1}^L \frac{2 \sum_{i=1}^N y^{(i)}_{l} \hat y_{l}^{(i)} }{ \sum_{i=1}^N (y^{(i)}_{l} + \hat y_{l}^{(i)}) } \\
    \text{example-F1} 
    &=
    \frac{1}{N} \sum_{i=1}^N \frac{2 \sum_{l=1}^L y^{(i)}_{l} \hat y_{l}^{(i)} }{ \sum_{l=1}^L (y^{(i)}_{l} + \hat y_{l}^{(i)}) }.
\end{align*}
}


These metrics compute either F1-score over the label matrix or averaged F1-score over targets or samples.

\subsubsection{Implementation Details}
For {\fairnotion}, we use $\gamma = 1, 5, 10$
to illustrate when it approximates DP or EOp. 
We also vary $\lambda$, the coefficient of fairness loss $\ell_{s_\gamma (\yv, \yv_{adv})}(h)$,
from $1$ to $5000$ to study the trade-off between fairness and accuracy (in terms of micro-, macro-, and example-F1). 
We randomly choose 70\% data for training and 30\% for testing. 
Other hyperparameters for MPVAE training such as batch size, epochs, and learning rates are fixed throughout all experiments.
A full list of hyperparameters used in this paper is provided in appendix \ref{app:hyperparameter}.

\subsection{Estimate DP and EOp with {\fairnotion}}
We first evaluate how well {\fairnotion} can approximate DP and EOp to answer RQ1. 
To do so, we train a MPVAE without any regularizers for 20 epochs on Adult and Credit datasets and evaluate how it violates DP and EOp.
We choose the largest label group (i.e., the label that appears most frequent) as the advantaged group. 
This allows us to construct a reliable estimate of EOp, which could be used as the ground truth. 

Figure \ref{fig:fairnotion-unify} shows how fairness violations estimated by {\fairnotion} change under different $\gamma$,
with DP and EOp marked on the left and right y-axis.
From the figure, the starting points of {\fairnotion} curves at $\gamma=0.1$ locate close to DP, 
and the ending points at $\gamma=10$ are close to EOp; 
these observations justify the effectiveness of {\fairnotion} in approximating DP and EOp, consistent with theoretical analysis.

\begin{figure}[tb]
    \centering
    \includesvg[pretex=\fontsize{4}{8}\selectfont,width=\linewidth]{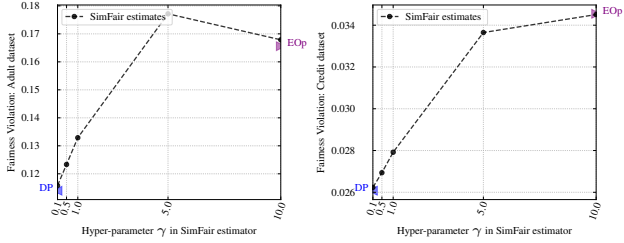}
    \caption{
    $s_\gamma$-SimFair can estimate DP and EOp with different hyperparameter $\gamma$, 
    DP and EOp estimates are marked on left and right y-axes.
    }
    \label{fig:fairnotion-unify}
    \vspace{-0.1in}
\end{figure}

Next, we study the robustness of three estimators
by varying the numbers of samples in the advantaged group to different levels and evaluating how estimates of DP and EOp change. 
Results summarized in Table \ref{tab:fairnotion-estimate-sub}
are averaged over 10 independent replications.
Empirically, EOp estimator degrades drastically as the size of \textit{observed} advantaged group decreases. 
{\fairnotion} with large $\gamma(=5)$, in contrary, produces more stable EOp estimates when EOp estimator fails. 
In terms of DP, both its own and {\fairnotion} estimator produce similarly stable results, 
which is reasonable as we only decrease the size of the advantaged group.

\begin{table}[tb]
    \centering
    \resizebox{0.85\linewidth}{!}{
    \begin{tabular}{rrccccccc}
    \toprule
    & $\yv_{adv}$ obs.($\%$)  & DP &  $ s_{ 0.1 } $-SF &  $ s_{0.5} $-SF &  $ s_{1} $-SF &  $ s_{5} $-SF & $ s_{10} $-SF &  EOp \\
    \cmidrule(l){2-9}
    \multirow{4}{*}{\rotatebox[origin=c]{90}{Adult}}
    & 100 \% & $0.11^*$     & 0.12  & 0.12  & 0.13  & 0.18  & 0.17  & $0.17^*$  \\
    & 70 \%  & 0.11         & 0.11  & 0.12  & 0.13  & 0.18  & 0.18  & 0.17      \\
    & 30 \%  & 0.10         & 0.10  & 0.11  & 0.12  & 0.17  & 0.18  & 0.17      \\
    & 10 \%  & 0.10         & 0.10  & 0.10  & 0.11  & 0.14  & 0.16  & 0.17      \\
    & 5 \%  & 0.10          & 0.10  & 0.10  & 0.11  & 0.15  & 0.23  & 0.27      \\
    \cmidrule(l){2-9}
    \multirow{4}{*}{\rotatebox[origin=c]{90}{Credit}}
    & 100 \% & $0.03^*$     & 0.03  & 0.03  & 0.03  & 0.03  &  0.03  & $0.03^*$  \\
    & 70 \%  & 0.03         & 0.03  & 0.03  & 0.03  & 0.03  &  0.04  & 0.04      \\
    & 30 \%  & 0.02         & 0.02  & 0.02  & 0.02  & 0.03  &  0.03  & 0.03      \\
    & 10 \%  & 0.02         & 0.02  & 0.02  & 0.02  & 0.03  &  0.04  & 0.04      \\
    & 5 \%   & 0.02         & 0.02  & 0.02  & 0.02  & 0.03  &  0.04  & 0.05      \\
    \bottomrule
    \end{tabular}
    }
    \caption{
    DP, EOp, and {\fairnotion} estimates (denoted as $s_\gamma$-SF) on Adult and Credit datasets. 
    Certain portions of samples in the advantaged group are kept (col. $\yv_{adv}$ obs.($\%$)) to check the robustness of different estimators. 
    Results are averaged over 10 replications. 
    Estimates of DP and EOp on $100\%$ portion of samples are considered as the \textit{ground truth} (marked with asterisk).  
    {\fairnotion} estimator is more robust than EOp estimator. 
    }
    \label{tab:fairnotion-estimate-sub}
\end{table}

\subsection{Performance of Regularization}

After showing that {\fairnotion} can approximate DP and EOp well,
we evaluate how well it can help achieve DP and EOp. 

We start with reporting fairness violations of MPVAE trained with DP, EOp, and {\fairnotion} regularizers.
On each dataset, two potential advantaged groups are considered.
The first group is the largest label group as in the last subsection, 
and the second group is chosen to be a \textit{small} label group but we can still estimate EOp on the test set. 
For Adult dataset, since it has more labels, we choose the 18-th largest label group, 
which is the smallest one that has more than 100 test samples from the advantaged group out of 152 possible labels.  
For Credit dataset, we choose the 9-th largest, 
this group has at least 10 test samples from the advantaged label out of 13 possible labels. 
Throughout experiments, we fix $\lambda = 10$ and run 10 replications to smooth out randomness. 

Table \ref{tab:fairnotion-perform} shows resultant DP and EOp achieved by different methods. 
In all experiments, {\fairnotion} performs competitive to DP regularizer
and better than EOp regularizer in terms of minimizing these metrics as objectives.
Notably, when the advantaged group is small, the vanilla EOp regularizer mitigates EOp violation poorly, 
but {\fairnotion} still reduces it significantly. 
Moreover, {\fairnotion} maintains a better DP-EOp balance, even they are known to be incompatible \cite{barocas2017fairness}. 
For example, $s_1$-SF regularzier helps achieve better DP and EOp simultaneously than a DP regularizer on the largest label group on Adult dataset.
We interpret this observation as a byproduct of the \textit{biased} estimation given by {\fairnotion}. 
As {\fairnotion} is \textit{biased} towards DP (EOp) when estimating EOp (DP), 
such bias implicitly considers the other metric and hence strikes a batter balance. 
These results clearly establish the power of {\fairnotion} in minimizing DP and EOp.

\begin{table}[htpb]
    \centering
    \resizebox{\linewidth}{!}{
    \begin{tabular}{rrr ccc ccc}
    \toprule
    & $|\yv_{adv}|$ & Metric & DP reg & $ s_{ 1 } $-SF reg & $ s_{5} $-SF reg & $ s_{10} $-SF reg & EOp reg & No reg\\
    \cmidrule(l){2-9}
    {\multirow{4}{*}{\rotatebox[origin=c]{90}{Adult}}} 
    & {\multirow{2}{*}{{No.1}}}   & DP  & 0.038 & {\bf 0.031 } & 0.038  & 0.043  & 0.045  & 0.111  \\
    &                             & EOp & 0.051 & 0.042  & {\bf0.030 } &  0.034  & 0.035  & 0.161  \\
    \cmidrule(l){3-9}
    & {\multirow{2}{*}{{No.18}}}  & DP  & {\bf 0.038 } & {\bf 0.038 } & 0.043  & 0.045  & 0.094  & 0.111  \\
    &                             & EOp & 0.076  & 0.072  & 0.037  & {\bf 0.027 } & 0.066  & 0.095  \\
    \cmidrule(l){2-9}
    {\multirow{4}{*}{\rotatebox[origin=c]{90}{Credit}}} 
    & {\multirow{2}{*}{{No.1}}}   & DP  & 0.018  & 0.018  & {\bf 0.017 } & 0.018  & 0.018  & 0.029  \\
    &                             & EOp & 0.026  & 0.026  & {\bf 0.025 } & {\bf 0.025 } & 0.026  & 0.038  \\
    \cmidrule(l){3-9}
    &{\multirow{2}{*}{{No.9}}}   & DP   & {\bf 0.018 } & {\bf 0.018 } & 0.019  & 0.019  & 0.030  & 0.030  \\
    &                            & EOp  & 0.202  & {\bf 0.192 } & 0.193  & 0.197  & 0.241  & 0.241  \\
    \bottomrule
    \end{tabular}
    }
    \caption{
    DP and EOp violations of MPVAE trained with DP, EOp, and {\fairnotion} regularziers. 
    On each dataset, a large and a small advantaged groups (measured by their ranking in col. $|\yv_{adv}|$) are tested. 
    Results are averaged over 10 replications, best results are in bold. 
    }
    \label{tab:fairnotion-perform}
\end{table}

To better reveal the limitation of EOp regularizer, 
we further evaluate how fair a model can be achieved by the use of different methods.
To do so, we choose a large $\lambda=5000$.
Note that such large $\lambda$, will be shown shortly, significantly impedes accuracy.
Here we sacrifice all accuracy to check the potential of different methods.

We run 3 replications on top 18 largest label groups in Adult dataset and
top 9 largest label groups in Credit dataset as advantaged group separately\footnote{As described above, these groups have sufficient test samples to check violations.}. 
Figure \ref{fig:fairnotion-all} shows resultant DP and EOp achieved by different methods. 
Compared to EOp regularizer, which performs the worst on all labels, {\fairnotion} is much more stable. 
In extreme cases, {\fairnotion}, as a good approximation of EOp, also encounter failure ultimately, 
but it is much more robust. 



\begin{figure}[htpb]
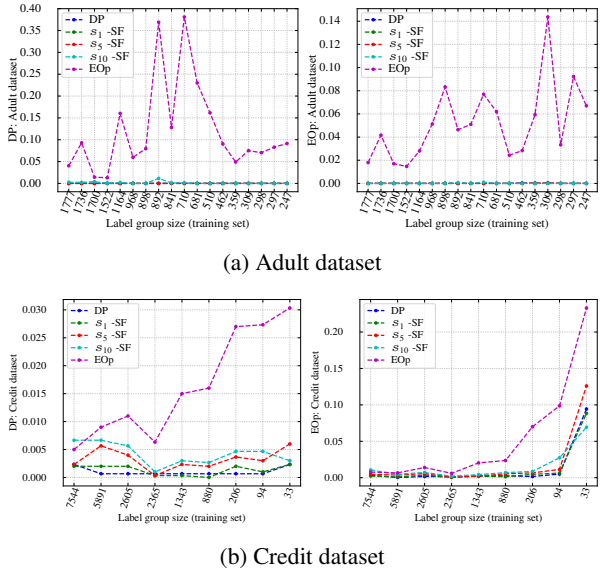

    \renewcommand{\t}[1]{
        $s_{\scalebox{0.8}{#1}}$
    }

    \def\figwidth{0.95\linewidth}
    \centering
    \begin{subfigure}[b]{\figwidth}
        \centering
        \includesvg[pretex=\fontsize{4.5}{8}\selectfont,width=\linewidth]{figures/fairnotion-all-adult.svg}
        \caption{Adult dataset}
    \end{subfigure}
    \\
    \begin{subfigure}[b]{\figwidth}
        \centering
        \includesvg[pretex=\fontsize{4}{8}\selectfont,width=\linewidth]{figures/fairnotion-all-credit.svg}
    \caption{Credit dataset}
    \end{subfigure}
    
    \caption{
    Achieved DP and EOp as the advantaged group becomes smaller. 
    An extremely large $\lambda=5000$ is used to enforce fairness mitigation. 
    Compared to EOp regularizer, {\fairnotion} is more robust to the sample size. 
    }
    \label{fig:fairnotion-all}
\end{figure}


\subsection*{Fairness-Accuracy Tradeoff}
\label{sec:eop-accuracy}

We end up this section with a study on the tradeoff between fairness and accuracy on the two label groups from the previous section.
Hyperparameter $\lambda$ varies from $1$ to $5000$ and results are averaged over 10 replications.
Due to the page limit, we only report EOp-accuracy tradeoffs on Credit dataset in Figure \ref{fig:fairnotion-tradeoff-eop-credit} here
and defer other figures to appendix \ref{app:tradeoffs}. 
Nevertheless,
conclusions drawn here apply to all experiments. 

Overall, micro- and example-F1 are much more robust to fairness requirement than macro-F1.
On Credit dataset, they are even improved slightly when a small fairness regularization is added. 
We hypothesize that fairness regularization indirectly adds smooth conditions and penalizes unstable predictions. 
{\fairnotion} has similar tradeoff patterns compared with the DP regularizer and does not encounter instability as EOp regularizer does.
In addition, on small label groups where EOp regularizer fails, 
its {\fairnotion} approximation succeeds in achieving low EOp violation,
and performs one of the best in handling tradeoffs. 

\begin{figure}[tb]
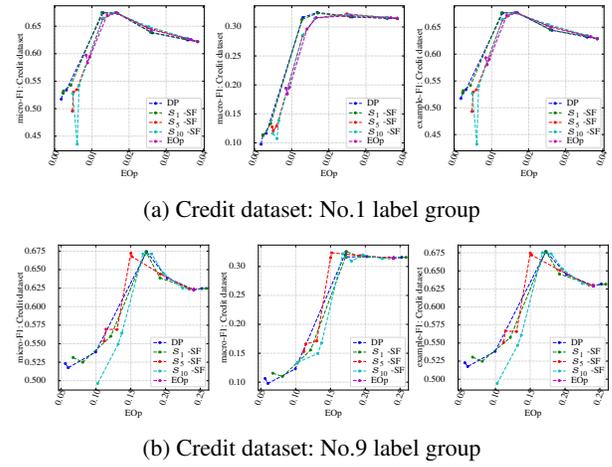

    \renewcommand{\t}[1]{
        $s_{\scalebox{0.8}{#1}}$
    }
    
    \def\figwidth{0.95\linewidth}
    
    
    \centering
    \begin{subfigure}[b]{\figwidth}%
        \centering%
        \includesvg[pretex=\fontsize{3.5}{8}\selectfont,width=\linewidth]{figures/credit-y1-eop.svg}
        \caption{Credit dataset: No.1 label group}
    \end{subfigure}
    
    \begin{subfigure}[b]{\figwidth}
        \centering
        \includesvg[pretex=\fontsize{3.5}{8}\selectfont,width=\linewidth]{figures/credit-y9-eop.svg}
        \caption{Credit dataset: No.9 label group}
    \end{subfigure}
    
    \caption{
    EOp-accuracy tradeoffs on Credit dataset. EOp regularizer is unstable and ineffective when the advantaged group is small,
    {\fairnotion}, on the other hand, preserves similar tradeoff trend as DP on both large and small label groups.
    }
 \label{fig:fairnotion-tradeoff-eop-credit}
\end{figure}

\section{Conclusions}
\label{sec:conclusion}

In this paper, We study the important problem of enforcing fairness on multi-label classification.
Given the ubiquitous imbalanced issue with respect to label groups, 
we propose {\fairnotion}, an effective framework that helps achieve existing group fairness metric: DP and EOp.
We first establish a formal extension of DP and EOp condition to multi-label scenarios, 
then prove that (extended) DP and EOp can be exactly expressed by {\fairnotion}, 
and can be approximated arbitrarily well. 
Experiments on two real-world datasets echos with theoretical analysis and reveals limitations of EOp regularizer. 
{\fairnotion}, in contrary, shows strong robustness against the challenges EOp regularizer cannot overcome. 

{\fairnotion} is a general tool. 
The concept and technique derived in this paper can be applied to multi-class classification as well, 
so long as a proper similarity function can be defined in the label space $\mathcal Y$. In the future, we plan to further conduct theoretical analysis on {\fairnotion} regularizer and convergence, and apply {\fairnotion} in a post-processing framework. 


\section*{Acknowledgement}

This work is supported in part by the US National Science Foundation under grant NSF IIS-2226108 and NSF IIS-2141037. Any opinions, findings, and conclusions or recommendations expressed in this material are those of the author(s) and do not necessarily reflect the views of the National Science Foundation.

\bibliography{ref.bib}

\begin{thebibliography}{48}
\providecommand{\natexlab}[1]{#1}

\bibitem[{Bai, Kong, and Gomes(2020)}]{bai2020mpvae}
Bai, J.; Kong, S.; and Gomes, C. 2020.
\newblock Disentangled Variational Autoencoder based Multi-Label Classification
  with Covariance-Aware Multivariate Probit Model.
\newblock \emph{Proceedings of the Twenty-Ninth International Joint Conference
  on Artificial Intelligence}.

\bibitem[{Barocas, Hardt, and Narayanan(2017)}]{barocas2017fairness}
Barocas, S.; Hardt, M.; and Narayanan, A. 2017.
\newblock Fairness in machine learning.
\newblock \emph{Nips tutorial}, 1: 2.

\bibitem[{Barocas and Selbst(2016)}]{barocas2016big}
Barocas, S.; and Selbst, A.~D. 2016.
\newblock Big data's disparate impact.
\newblock \emph{Calif. L. Rev.}, 104: 671.

\bibitem[{Boutell et~al.(2004)Boutell, Luo, Shen, and
  Brown}]{boutell2004learning}
Boutell, M.~R.; Luo, J.; Shen, X.; and Brown, C.~M. 2004.
\newblock Learning multi-label scene classification.
\newblock \emph{Pattern recognition}, 37(9): 1757--1771.

\bibitem[{Buolamwini and Gebru(2018)}]{buolamwini2018gender}
Buolamwini, J.; and Gebru, T. 2018.
\newblock Gender shades: Intersectional accuracy disparities in commercial
  gender classification.
\newblock In \emph{Conference on fairness, accountability and transparency},
  77--91. PMLR.

\bibitem[{Calders, Kamiran, and Pechenizkiy(2009)}]{calders2009building}
Calders, T.; Kamiran, F.; and Pechenizkiy, M. 2009.
\newblock Building classifiers with independency constraints.
\newblock In \emph{2009 IEEE International Conference on Data Mining
  Workshops}, 13--18. IEEE.

\bibitem[{Chen, Xue, and Gomes(2018)}]{chen2018end}
Chen, D.; Xue, Y.; and Gomes, C. 2018.
\newblock End-to-end learning for the deep multivariate probit model.
\newblock In \emph{International Conference on Machine Learning}, 932--941.
  PMLR.

\bibitem[{Chib and Greenberg(1998)}]{chib1998analysis}
Chib, S.; and Greenberg, E. 1998.
\newblock Analysis of multivariate probit models.
\newblock \emph{Biometrika}, 85(2): 347--361.

\bibitem[{Chouldechova and Roth(2020)}]{chouldechova2020snapshot}
Chouldechova, A.; and Roth, A. 2020.
\newblock A snapshot of the frontiers of fairness in machine learning.
\newblock \emph{Communications of the ACM}, 63(5): 82--89.

\bibitem[{Clare and King(2001)}]{clare2001knowledge}
Clare, A.; and King, R.~D. 2001.
\newblock {Knowledge Discovery in Multi-label Phenotype Data}.
\newblock In \emph{European conference on principles of data mining and
  knowledge discovery}, 42--53. Springer.

\bibitem[{Creager et~al.(2019)Creager, Madras, Jacobsen, Weis, Swersky,
  Pitassi, and Zemel}]{creager2019flexibly}
Creager, E.; Madras, D.; Jacobsen, J.-H.; Weis, M.~A.; Swersky, K.; Pitassi,
  T.; and Zemel, R. 2019.
\newblock Flexibly Fair Representation Learning by Disentanglement.
\newblock arXiv:1906.02589.

\bibitem[{Dekel and Shamir(2010)}]{dekel2010multi}
Dekel, O.; and Shamir, O. 2010.
\newblock Multiclass-Multilabel Classification with More Classes than Examples.
\newblock In Teh, Y.~W.; and Titterington, M., eds., \emph{Proceedings of the
  Thirteenth International Conference on Artificial Intelligence and
  Statistics}, volume~9 of \emph{Proceedings of Machine Learning Research},
  137--144. Chia Laguna Resort, Sardinia, Italy: PMLR.

\bibitem[{Dressel and Farid(2018)}]{dressel2018accu}
Dressel, J.; and Farid, H. 2018.
\newblock The accuracy, fairness, and limits of predicting recidivism.
\newblock \emph{Science Advances}, 4(1): eaao5580.

\bibitem[{Dwork et~al.(2011)Dwork, Hardt, Pitassi, Reingold, and
  Zemel}]{dwork2011fairness}
Dwork, C.; Hardt, M.; Pitassi, T.; Reingold, O.; and Zemel, R. 2011.
\newblock Fairness Through Awareness.
\newblock arXiv:1104.3913.

\bibitem[{Edwards and Storkey(2015)}]{edwards2015censoring}
Edwards, H.; and Storkey, A. 2015.
\newblock {Censoring Representations with an Adversary}.
\newblock \emph{arXiv preprint arXiv:1511.05897}.

\bibitem[{El~Kafrawy, Mausad, and Esmail(2015)}]{el2015experimental}
El~Kafrawy, P.; Mausad, A.; and Esmail, H. 2015.
\newblock Experimental comparison of methods for multi-label classification in
  different application domains.
\newblock \emph{International Journal of Computer Applications}, 114(19): 1--9.

\bibitem[{F{\"u}rnkranz et~al.(2008)F{\"u}rnkranz, H{\"u}llermeier,
  Loza~Menc{\'\i}a, and Brinker}]{furnkranz2008multilabel}
F{\"u}rnkranz, J.; H{\"u}llermeier, E.; Loza~Menc{\'\i}a, E.; and Brinker, K.
  2008.
\newblock Multilabel classification via calibrated label ranking.
\newblock \emph{Machine learning}, 73(2): 133--153.

\bibitem[{Gong et~al.(2019)Gong, Liu, Chu, and Yu}]{gong2019using}
Gong, T.; Liu, B.; Chu, Q.; and Yu, N. 2019.
\newblock Using multi-label classification to improve object detection.
\newblock \emph{neurocomputing}, 370: 174--185.

\bibitem[{Hardt, Price, and Srebro(2016)}]{hardt2016equality}
Hardt, M.; Price, E.; and Srebro, N. 2016.
\newblock Equality of Opportunity in Supervised Learning.
\newblock arXiv:1610.02413.

\bibitem[{Kamishima et~al.(2012)Kamishima, Akaho, Asoh, and
  Sakuma}]{kamishima2012fairness}
Kamishima, T.; Akaho, S.; Asoh, H.; and Sakuma, J. 2012.
\newblock Fairness-aware classifier with prejudice remover regularizer.
\newblock In \emph{Joint European conference on machine learning and knowledge
  discovery in databases}, 35--50. Springer.

\bibitem[{Kingma and Ba(2014)}]{kingma2014adam}
Kingma, D.~P.; and Ba, J. 2014.
\newblock {Adam: A method for stochastic optimization}.
\newblock \emph{arXiv preprint arXiv:1412.6980}.

\bibitem[{Kingma and Welling(2014)}]{kingma2014autoencoding}
Kingma, D.~P.; and Welling, M. 2014.
\newblock Auto-Encoding Variational Bayes.
\newblock arXiv:1312.6114.

\bibitem[{Kohavi(1996)}]{kohavi1996scale}
Kohavi, R. 1996.
\newblock Scaling up the Accuracy of Naive-Bayes Classifiers: A Decision-Tree
  Hybrid.
\newblock In \emph{Proceedings of the Second International Conference on
  Knowledge Discovery and Data Mining}, KDD'96, 202–207. AAAI Press.

\bibitem[{Liu et~al.(2021)Liu, Wang, Shen, and Tsang}]{liu2021emerge}
Liu, W.; Wang, H.; Shen, X.; and Tsang, I. 2021.
\newblock The Emerging Trends of Multi-Label Learning.
\newblock \emph{IEEE Transactions on Pattern Analysis and Machine
  Intelligence}, 1–1.

\bibitem[{Locatello et~al.(2019)Locatello, Abbati, Rainforth, Bauer,
  Schölkopf, and Bachem}]{locatello2019fairness}
Locatello, F.; Abbati, G.; Rainforth, T.; Bauer, S.; Schölkopf, B.; and
  Bachem, O. 2019.
\newblock On the Fairness of Disentangled Representations.
\newblock arXiv:1905.13662.

\bibitem[{Madras et~al.(2018)Madras, Creager, Pitassi, and
  Zemel}]{madras2018learning}
Madras, D.; Creager, E.; Pitassi, T.; and Zemel, R. 2018.
\newblock {Learning Adversarially Fair and Transferable Representations}.
\newblock In \emph{International Conference on Machine Learning}, 3384--3393.
  PMLR.

\bibitem[{Mary, Calauzenes, and El~Karoui(2019)}]{mary2019fairness}
Mary, J.; Calauzenes, C.; and El~Karoui, N. 2019.
\newblock Fairness-aware learning for continuous attributes and treatments.
\newblock In \emph{International Conference on Machine Learning}, 4382--4391.
  PMLR.

\bibitem[{Mohler et~al.(2018)Mohler, Raje, Carter, Valasik, and
  Brantingham}]{mohler2018penalized}
Mohler, G.; Raje, R.; Carter, J.; Valasik, M.; and Brantingham, J. 2018.
\newblock A penalized likelihood method for balancing accuracy and fairness in
  predictive policing.
\newblock In \emph{2018 IEEE international conference on systems, man, and
  cybernetics (SMC)}, 2454--2459. IEEE.

\bibitem[{Nam et~al.(2014)Nam, Kim, Loza~Menc{\'\i}a, Gurevych, and
  F{\"u}rnkranz}]{nam2014large}
Nam, J.; Kim, J.; Loza~Menc{\'\i}a, E.; Gurevych, I.; and F{\"u}rnkranz, J.
  2014.
\newblock Large-scale multi-label text classification—revisiting neural
  networks.
\newblock In \emph{Joint european conference on machine learning and knowledge
  discovery in databases}, 437--452. Springer.

\bibitem[{Pedreshi, Ruggieri, and Turini(2008)}]{pedreshi2008discrimination}
Pedreshi, D.; Ruggieri, S.; and Turini, F. 2008.
\newblock Discrimination-aware data mining.
\newblock In \emph{Proceedings of the 14th ACM SIGKDD international conference
  on Knowledge discovery and data mining}, 560--568.

\bibitem[{Petersen et~al.(2021)Petersen, Mukherjee, Sun, and
  Yurochkin}]{petersen2021post}
Petersen, F.; Mukherjee, D.; Sun, Y.; and Yurochkin, M. 2021.
\newblock Post-processing for Individual Fairness.
\newblock \emph{Advances in Neural Information Processing Systems}, 34.

\bibitem[{Pleiss et~al.(2017)Pleiss, Raghavan, Wu, Kleinberg, and
  Weinberger}]{pleiss2017fairness}
Pleiss, G.; Raghavan, M.; Wu, F.; Kleinberg, J.; and Weinberger, K.~Q. 2017.
\newblock On Fairness and Calibration.
\newblock arXiv:1709.02012.

\bibitem[{Read et~al.(2011)Read, Pfahringer, Holmes, and
  Frank}]{read2011classifier}
Read, J.; Pfahringer, B.; Holmes, G.; and Frank, E. 2011.
\newblock Classifier chains for multi-label classification.
\newblock \emph{Machine learning}, 85(3): 333--359.

\bibitem[{Reddy et~al.(2021)Reddy, Sharma, Mehri, Romero-Soriano, Shabanian,
  and Honari}]{reddy2021benchmarking}
Reddy, C.; Sharma, D.; Mehri, S.; Romero-Soriano, A.; Shabanian, S.; and
  Honari, S. 2021.
\newblock Benchmarking bias mitigation algorithms in representation learning
  through fairness metrics.
\newblock In \emph{Thirty-fifth Conference on Neural Information Processing
  Systems Datasets and Benchmarks Track (Round 1)}.

\bibitem[{Scutari, Panero, and Proissl(2021)}]{scutari2021achieving}
Scutari, M.; Panero, F.; and Proissl, M. 2021.
\newblock Achieving Fairness with a Simple Ridge Penalty.
\newblock arXiv:2105.13817.

\bibitem[{Tsoumakas, Katakis, and Vlahavas(2006)}]{tsoumakas2006review}
Tsoumakas, G.; Katakis, I.; and Vlahavas, I. 2006.
\newblock A review of multi-label classification methods.
\newblock In \emph{Proceedings of the 2nd ADBIS workshop on data mining and
  knowledge discovery (ADMKD 2006)}, 99--109. Citeseer.

\bibitem[{Tsoumakas and Vlahavas(2007)}]{tsoumakas2007random}
Tsoumakas, G.; and Vlahavas, I. 2007.
\newblock Random k-labelsets: An ensemble method for multilabel classification.
\newblock In \emph{European conference on machine learning}, 406--417.
  Springer.

\bibitem[{Woodworth et~al.(2017)Woodworth, Gunasekar, Ohannessian, and
  Srebro}]{woodworth2017learning}
Woodworth, B.; Gunasekar, S.; Ohannessian, M.~I.; and Srebro, N. 2017.
\newblock Learning non-discriminatory predictors.
\newblock In \emph{Conference on Learning Theory}, 1920--1953. PMLR.

\bibitem[{Wu and Zhou(2017)}]{wu2017unified}
Wu, X.-Z.; and Zhou, Z.-H. 2017.
\newblock A unified view of multi-label performance measures.
\newblock In \emph{international conference on machine learning}, 3780--3788.
  PMLR.

\bibitem[{Yang et~al.(2009)Yang, Sun, Wang, and Chen}]{yang2009effective}
Yang, B.; Sun, J.-T.; Wang, T.; and Chen, Z. 2009.
\newblock Effective multi-label active learning for text classification.
\newblock In \emph{Proceedings of the 15th ACM SIGKDD international conference
  on Knowledge discovery and data mining}, 917--926.

\bibitem[{Yeh and Lien(2009)}]{yeh2009comparisons}
Yeh, I.-C.; and Lien, C.-h. 2009.
\newblock The comparisons of data mining techniques for the predictive accuracy
  of probability of default of credit card clients.
\newblock \emph{Expert systems with applications}, 36(2): 2473--2480.

\bibitem[{Zafar et~al.(2017)Zafar, Valera, Gomez~Rodriguez, and
  Gummadi}]{zafar2017fairness}
Zafar, M.~B.; Valera, I.; Gomez~Rodriguez, M.; and Gummadi, K.~P. 2017.
\newblock Fairness beyond disparate treatment \& disparate impact: Learning
  classification without disparate mistreatment.
\newblock In \emph{Proceedings of the 26th international conference on world
  wide web}, 1171--1180.

\bibitem[{Zhang et~al.(2020)Zhang, Zhao, Duan, Chen, Zhang, and
  Tang}]{zhang2020multi}
Zhang, D.; Zhao, S.; Duan, Z.; Chen, J.; Zhang, Y.; and Tang, J. 2020.
\newblock A multi-label classification method using a hierarchical and
  transparent representation for paper-reviewer recommendation.
\newblock \emph{ACM Transactions on Information Systems (TOIS)}, 38(1): 1--20.

\bibitem[{Zhang et~al.(2018)Zhang, Li, Liu, and Geng}]{zhang2018binary}
Zhang, M.-L.; Li, Y.-K.; Liu, X.-Y.; and Geng, X. 2018.
\newblock Binary relevance for multi-label learning: an overview.
\newblock \emph{Frontiers of Computer Science}, 12(2): 191--202.

\bibitem[{Zhang and Zhou(2007)}]{zhang2007ml}
Zhang, M.-L.; and Zhou, Z.-H. 2007.
\newblock {Ml-knn: A Lazy Learning Approach to Multi-Label Learning}.
\newblock \emph{Pattern recognition}, 40(7): 2038--2048.

\bibitem[{Zhang and Zhou(2014)}]{zhang2014review}
Zhang, M.-L.; and Zhou, Z.-H. 2014.
\newblock A Review on Multi-Label Learning Algorithms.
\newblock \emph{IEEE Transactions on Knowledge and Data Engineering}, 26(8):
  1819--1837.

\bibitem[{Zhao et~al.(2020)Zhao, Guo, Shen, and Ye}]{zhao2020adaptive}
Zhao, Z.; Guo, Y.; Shen, H.; and Ye, J. 2020.
\newblock Adaptive object detection with dual multi-label prediction.
\newblock In \emph{European Conference on Computer Vision}, 54--69. Springer.

\bibitem[{Zheng, Mobasher, and Burke(2014)}]{zheng2014context}
Zheng, Y.; Mobasher, B.; and Burke, R. 2014.
\newblock Context recommendation using multi-label classification.
\newblock In \emph{2014 IEEE/WIC/ACM International Joint Conferences on Web
  Intelligence (WI) and Intelligent Agent Technologies (IAT)}, volume~2,
  288--295. IEEE.

\end{thebibliography}

\clearpage
\appendix
\onecolumn

\setlength{\floatsep}{12.0pt plus 2.0pt minus 2.0pt} 
\setlength{\textfloatsep}{20.0pt plus 2.0pt minus 4.0pt} 
\setlength{\intextsep}{12.0pt plus 2.0pt minus 2.0pt} 
\setlength{\dbltextfloatsep}{20.0pt plus 2.0pt minus 4.0pt} 
\setlength{\dblfloatsep}{12.0pt plus 2.0pt minus 2.0pt} 
\setlength{\abovecaptionskip}{10pt} 
\setlength{\belowcaptionskip}{0pt} 

\setlength{\abovedisplayskip}{2ex} 
\setlength{\belowdisplayskip}{2ex} 
\setlength{\arraycolsep}{0.5em} 
\renewcommand\mathloose
{%
    \thinmuskip=3mu   
    \medmuskip=4mu    
    \thickmuskip=5mu plus 5mu  
}
\mathloose


\section{Notation Table}
\label{app:notation}

We summarize notations used in this paper in table \ref{tab:notation}. 

\begin{table}[htp]
    \centering
    \begin{tabular}{l l}
    \toprule
    Notations & Meaning \\
    \midrule
    $\xv \in \mathcal X$                                 & Non-sensitive feature and
                                                         non-sensitive feature space.\\
    $a \in \mathcal A$                                   & Sensitive feature and
                                                         sensitive feature space.\\
    $\yv \in \mathcal Y=\{0, 1\}^L$                      & Label and label space.\\
    $h = f \circ g: \mathcal{X} \rightarrow \mathcal Y $ & A composited multi-label classifier, 
                                                         $f: \mathcal X \rightarrow [0, 1]^L$ and $g: [0, 1]^L \rightarrow \mathcal Y$. \\
    $\hat{\yv} = h(\xv) = g(f(x))$                       & Predicted label (Prediction). \\
    $\tilde \yv = f(\xv)$                                & Predicted probability vector. \\
    $s_\gamma(\yv, \yv')$                                & Similarity between label $\yv$ and $\yv'$.\\
    $\gamma \geq 0$                                      & Scaling hyperparameter in similarity. \\
    $\lambda \geq 0$                                     & Coefficient of fairness penalty. \\
    $\ell_{mlc}(h)$                                      & Multi-label classification loss.\\
    $\hat \ell_{s_\gamma (\yv, \yv_{adv})}(h)$           & {\fairnotion} violation (penalty). \\
    \bottomrule
    \end{tabular}
    \caption{Main notations used in this paper.
    }
    \label{tab:notation}
\end{table}

\section{Omitted Proofs}
\subsection{Proof of Proposition \ref{thm:dp-eop-mlc}}
\label{app:dp-eop-mlc}


\begin{proposition}[DP and EOp condition in MLC]
\label{thm:dp-eop-mlc-app}
For a multi-label classifier that takes the form $h = f \circ g$, 
where $\tilde \yv = g(\xv)$ is the predicted probability and $\hat \yv = f(\tilde \yv)$ is computed elementwisely, 
DP and EOp hold if for any $ k \in \mathcal A$
\begin{align}
    \text{DP: } 
    &\quad 
    \E [ \tilde \yv \mid a=k ] = \E [\tilde \yv ] \notag \\
    \text{EOp: }
    &\quad
    \E [ \tilde \yv \mid a=k, \yv=\yv_{adv}] = \E [\tilde \yv \mid \yv=\yv_{adv}]. \label{eq:eop-app}
\end{align}
\end{proposition}

\begin{proof}
    Here we derive the condition of EOp in eqn~\eqref{eq:eop-app}, the proof for DP can be obtained in the same way. 
    Note that the conditional distribution of prediction $\yv$ in $k$-th demographic subgroup of the advantaged group is given by
    \begin{align*}
        & p(\hat \yv \mid a = k, \yv = \yv_{adv}) 
        = 
        \int p(\hat \yv, \xv \mid a = k, \yv = \yv_{adv}) \dif \xv \\
        &=
        \int p(\hat \yv \mid \xv, a = k, \yv = \yv_{adv}) p(\xv \mid a=k, \yv = \yv_{adv}) \dif \xv \\
        &\overset{(a)}{=}
        \int p(\hat \yv \mid \xv) p(\xv \mid a=k, \yv = \yv_{adv}) \dif \xv \\
        &= 
        \int g(\xv) p(\xv \mid a = k, \yv = \yv_{adv}) \dif \xv \\
        &= 
        \E [ \tilde \yv \mid a=k, \yv=\yv_{adv}],
    \end{align*}
    where $(a)$ 
    holds because of the conditional independence $\hat \yv \perp (\yv, a) \mid \xv $. 
    Similar derivation gives us 
    \begin{align}
    p(\hat \yv \mid \yv = \yv_{adv}) = \E [\tilde \yv \mid \yv = \yv_{adv} ]. \notag
    \end{align}
    This indicates that the conditional independence requirement in EOp can be fulfilled if corresponding conditional expectations match.
\end{proof}

\subsection{Proof of Proposition \ref{thm:fairnotion-special}}
\label{app:fairnotion-special}


\begin{proposition}[DP and EOp are special cases of {\fairnotion}]

\label{thm:fairnotion-special-app}
Consider {\fairnotion} defined in eqn~\eqref{eq:fairnotion}, 
if similarity $s$ is a constant function $s(\yv, \yv')  = c$ for some $c$, 
then {\fairnotion} implies DP;
if $s$ is an indicator function $s(\yv, \yv') = \ones(\yv = \yv')$, 
then {\fairnotion} implies EOp.

\end{proposition}

\begin{proof}
The EOp case can be seen by taking special $s(\yv, \yv') = \ones(\yv = \yv')$ in eqn~\eqref{eq:fairnotion}.
To prove the DP case, 
note that for constant function $s(\yv, \yv')  = c$, 
the left hand side of eqn~\eqref{eq:fairnotion} becomes $\E [c \tilde{ \yv}] / \E [c] = \E [\tilde{ \yv}]$, 
and the right hand side is

\begin{align*}
    & \frac{\E [c \tilde{\yv} \ones(a=k)]}{ \E [c \ones(a=k)]}
    = 
    \frac{\E [\tilde{\yv} \ones(a=k)]}{ \E [\ones(a=k)]} \\
    &= \frac{1}{P(a=k)} \iint  \tilde{\yv} \ones(a=k) p(\tilde{\yv}, a) \dif a \tilde \yv \\
    &= \int \tilde{\yv} \int \frac{1}{P(a=k)} p(\ones(a=k) p(\tilde{\yv}, a) ) \dif a \dif \tilde{\yv} \\
    &= \int \tilde{\yv} p(\tilde{\yv} \mid a=k) \dif \tilde{\yv} = \E [\tilde{\yv} \mid a = k].
\end{align*}
Now eqn~\eqref{eq:fairnotion} requires $\E [\tilde{ \yv} \mid a = k] = \E [\tilde{ \yv}]$ for all $k$, 
which is exactly the condition of DP. 
\end{proof}

\subsection{Proof of Proposition \ref{thm:fairnotion-unify}}
\label{app:fairnotion-unify}


\begin{proposition}[{\fairnotion} helps achieve DP and EOp]
\label{thm:fairnotion-unify-app}
For any multi-label classifier $h$ satisfying {\fairnotion}, 
its violation of DP will be arbitrarily small if $\gamma$ is sufficiently small; 
and its violation of EOp will be arbitrarily small if $\gamma$ is sufficiently large. 
More generally, 
its violation of DP is arbitrarily close to its violation of {\fairnotion} for sufficiently small $\gamma$, 
and its violation of EOp is arbitrarily close to its violation of {\fairnotion} for sufficiently large $\gamma$. 
\end{proposition}

\begin{proof}
Let $\yv, \yv'$ be two arbitrary label vectors,
and $s^*(\yv, \yv')$ denote the limit of their similarity
$s_\gamma(\yv, \yv')$ as $\gamma \rightarrow 0$ or $\gamma \rightarrow \infty$.
Equivalently speaking, $s^*(\yv, \yv')$ is constant 1 or indicator function $\ones(\yv = \yv')$, 
which are special $s_\gamma$ by Proposition \ref{thm:fairnotion-special}. 
This allows us to express and upper bound the difference between violations of DP (or EOp) and {\fairnotion} on subgroup with $a=k$ and label $\yv_{adv}$ 
defined in eqn~\eqref{eq:fair-violation-k}
by 
{
\begin{align*}
&\| \ell_{s^*(\yv, \yv_{adv})} (f) - \ell_{s_\gamma(\yv, \yv_{adv})} (f)  \| \\
=& 
\Bigg\| \Big\| \frac{\E [\tilde{\yv} s^*(\yv, \yv_{adv})]}{\E [s^*(\yv, \yv_{adv})]} - \frac{\E [\tilde{\yv}  \ones(a=k) s^*(\yv, \yv_{adv})]}{\E [\ones(a=k) s^*(\yv, \yv_{adv})]} \Big\| - 
\Big\|
  \frac{\E [\tilde{\yv} s_\gamma( \yv, \yv_{adv})]}{\E [s_\gamma( \yv, \yv_{adv})]} - \frac{\E [\tilde{\yv}  \ones(a=k) s_\gamma( \yv, \yv_{adv})]}{\E [\ones(a=k) s_\gamma( \yv, \yv_{adv})]} \Big\| \Bigg\| \\
&\overset{(a)}{\leq}
\Bigg\| 
\left(\frac{\E [\tilde{\yv} s^*(\yv, \yv_{adv})]}{\E [s^*(\yv, \yv_{adv})]}  - \frac{\E [\tilde{\yv}  \ones(a=k) s^*(\yv, \yv_{adv})]}{\E [\ones(a=k) s^*(\yv, \yv_{adv})]} \right) - 
\left( \frac{\E [\tilde{\yv} s_\gamma( \yv, \yv_{adv})]}{\E [s_\gamma( \yv, \yv_{adv})]}
- \frac{\E [\tilde{\yv}  \ones(a=k) s_\gamma( \yv, \yv_{adv})]}{\E [\ones(a=k) s_\gamma( \yv, \yv_{adv})]} \right) \Bigg\| 
\\
&\overset{(b)}{\leq}
\Bigg\| \frac{\E [\tilde{\yv} s^*(\yv, \yv_{adv})]}{\E [s^*(\yv, \yv_{adv})]} - \frac{\E [\tilde{\yv} s_\gamma( \yv, \yv_{adv})]}{\E [s_\gamma( \yv, \yv_{adv})]} \Bigg\| +
\Bigg\| \frac{\E [\tilde{\yv}  \ones(a=k) s^*(\yv, \yv_{adv})]}{\E [\ones(a=k) s^*(\yv, \yv_{adv})]} - \frac{\E [\tilde{\yv}  \ones(a=k) s_\gamma( \yv, \yv_{adv})]}{\E [\ones(a=k) s_\gamma( \yv, \yv_{adv})]}\Bigg\|,
\label{eq:tri-ineq-thm2}
\end{align*}
}
where (a) holds from the reverse triangle inequality
and (b) holds from the triangle inequality. 

Next, sequence $s_\gamma(\yv, \yv')$ converges to $s^*(\yv, \yv')$ monotonically, we have the following \textit{almost surely} convergence
\begin{align*}
\tilde{\yv} s_\gamma(\yv, \yv_{adv}) \xrightarrow{a.s.} \tilde{\yv} s^*(\yv, \yv_{adv}).
\end{align*}


Recall that $\tilde{\yv}$ is the predicted probability vector so $\E [\| \Tilde{\yv} \| ] < \infty$, and $s_\gamma(\yv, \yv_{adv}), s^*(\yv, \yv_{adv}) \in [0, 1]$, we have 
\begin{align*}
&\max(
\| \tilde{\yv} s_\gamma(\yv, \yv_{adv}) \|,
\| \tilde{\yv} s^*(\yv, \yv_{adv}) \| )
\leq \| \tilde{\yv} \| 
\end{align*}

Let $\gamma^*$ denote 0 (for DP) or $\infty$ (for EOp). According to Dominated Convergence Theorem 
\begin{align*}
\lim_{\gamma \rightarrow \gamma^*} 
\| \E [\tilde{\yv} s_\gamma(\yv, \yv_{adv})] 
- \E [\tilde{\yv} s^*(\yv, \yv_{adv})] \| = 0,
\end{align*}

and 
\begin{align*}
\lim_{\gamma \rightarrow \gamma^*}  
\| \E [s_\gamma(\yv, \yv_{adv})] 
- \E [s^*(\yv, \yv_{adv})] \| = 0 .
\end{align*}

Without loss of generality we further assume $\E [s_\gamma(\yv, \yv_{adv})]$ and $\E [s^*(\yv, \yv_{adv})]$ are positive
\footnote{This only rules out case $p(\yv_{adv}) = 0$ when $\mathcal{Y}$ is discrete, where fairness concern does not exist.}. This gives us
\begin{align}
    \frac{\E [\tilde{\yv} s_\gamma(\yv, \yv_{adv})]}{\E [s_\gamma(\yv, \yv_{adv})]} \rightarrow \frac{\E [\tilde{\yv} s^*(\yv, \yv_{adv})]}{\E [s^*(\yv, \yv_{adv})]}.
\end{align}

In other words, the first term in eqn \eqref{eq:tri-ineq-thm2} can be bounded with arbitrarily small $\varepsilon$ as $\gamma \rightarrow 0 $ (for DP case with $s^*(\yv, \yv') = 1$)
or $\gamma \rightarrow \infty$ (for EOp case with $s^*(\yv, \yv') = \ones(\yv= \yv')$). 
Similar upper bound can be derived for the second term in the same way. 
Put together, we have 
\begin{align}
\| \ell_{s^*(\yv, \yv_{adv})} (f) - \ell_{s_\gamma(\yv, \yv_{adv})} (f)  \|
     \leq 2 \varepsilon,
\end{align}
where $\varepsilon$ depends on $\gamma$ and can be made arbitrarily small. 
The first part of the proposition is proved by further assuming $\ell_{s_\gamma}(\yv, \yv_{adv}) = 0$, i.e., 
\begin{align}
    \frac{\E [\tilde{\yv} s_\gamma(\yv, \yv_{adv})]}{\E [s_\gamma\yv, \yv_{adv})]}
    = \frac{\E [\tilde{\yv}  \ones(a=k) s_\gamma(\yv, \yv_{adv})]}{\E [\ones(a=k) s_\gamma(\yv, \yv_{adv})]}.
\end{align}

This completes our proof.

\end{proof}

\section{Algorithm}
\label{app:algorithm}

Here we provide a concise illustration of MPVAE traininig with {\fairnotion} regularizer.
In each step, we take one minibatch randomly selected from the dataset and update MPVAE with 
loss defined in eqn \eqref{eq:final-loss}. 
Algorithm \ref{algo:fairnotion} summarizes this step.
 
\begin{algorithm}[htb]
    \caption{One update of MPVAE with {\fairnotion} regularizer.}
    \label{algo:fairnotion}
     \textbf{Input}: 
        mini-batch
        $\{ (\xv^{(i)}, a^{(i)}, \yv^{(i)}) \}_{i=1}^n$,
        advantaged group $\yv_{adv}$,
        MPVAE $h$,
        hyperparameters $\gamma$ and $\lambda$.
        \\
    \textbf{Output}:
        Updated MPVAE $h$
    \begin{algorithmic}[1]
      \State Compute the empirical multi-label classification loss on the minibatch 
      \begin{align*}
          \hat \ell_{mlc} = \frac{1}{n} \sum_{i=1}^n \ell_{mlc}(\xv^{(i)}, \yv^{(i)}; h).
      \end{align*}
        \State For each sample, compute predictions from label and feature branches $\tilde \yv^{(i)}_{\yv}$, $\tilde \yv^{(i)}_{\xv}$, and $s^{(i)}_\gamma = s_\gamma(\yv^{(i)}, \yv_{adv})$.
        \State Compute empirical fairness loss $\hat \ell_{s_\gamma (\yv, \yv_{adv})}$ on the minibatch with eqn \eqref{eq:fair-violation-k} or \eqref{eq:fair-violation-k2}. Estimate each term by eqn \eqref{eq:fairnotion-est-all} or \eqref{eq:fairnotion-est-sub}.
        \State Take one updates on $h$ with Adam \cite{kingma2014adam} to minimize the final empirical loss
        \begin{align*}
            \hat \ell = \hat \ell_{mlc} + \lambda \hat \ell_{s_\gamma (\yv, \yv_{adv})}
        \end{align*}
    \end{algorithmic}
\end{algorithm}

\section{Experimental Details}
\label{app:hyperparameter}

Here we present the hyperparameters we used in experiments for reproducibility. 
For experiments run 10 replications, we used random seeds from 1 to 10;
for experiments that only had 3 replications, we used seeds from 1 to 3. 
We used the same fixed hyperparameters except $\lambda$ and $\gamma$ throughout all experiments. 
We made one modification to MPVAE training by clipping the gradient norm to stabilize the training, 
other training strategies are adopted from \citet{bai2020mpvae} and details can be found therein. 
Hyperparameters that are different from \citet{bai2020mpvae} are listed in table \ref{tab:hyperparameter}.


\begin{table}[htb]
    \centering
    \begin{tabular}{llll}
    \toprule
    Epochs & Ranking loss coefficient & Latent dimension & Max. gradient norm \\
    \midrule
    20     & 100                      & 32               & 5 \\
    \bottomrule
    \end{tabular}
    \caption{
    Hyperparameter settings of our experiments. 
    Other hyperparameters are adopted from \citet{bai2020mpvae}.
    }
    \label{tab:hyperparameter}
\end{table}

\clearpage
\section{More Experimental Results}
\subsection{Fairness-Accuracy Tradeoff}
\label{app:tradeoffs}

In this section we show full EOp- and DP-accuracy tradeoff in Figure \ref{fig:fairnotion-tradeoff-eop} and \ref{fig:fairnotion-tradeoff-dp} following the same logic as in Section \ref{sec:eop-accuracy}.
Note that whereas EOp-accuracy tradeoff on Adult dataset has different curvatures, but the conclusion does not change. 
DP-accuracy tradeoff has the similar trends as the EOp-accuracy tradeoff, so we omit reiterating observations.

\begin{figure*}[htb]
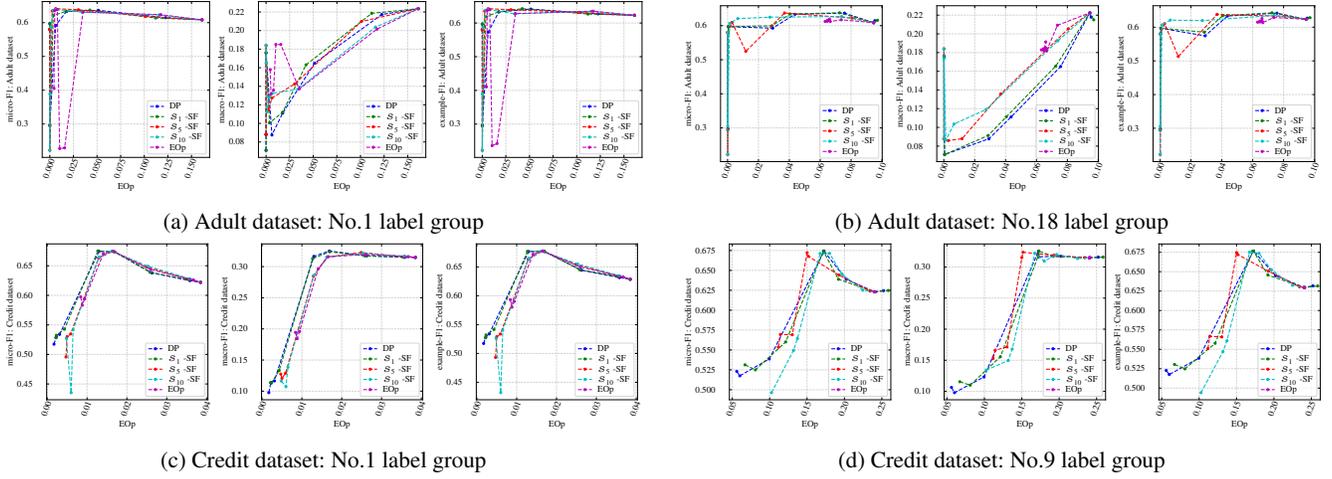

    \renewcommand{\t}[1]{
        $s_{\scalebox{0.8}{#1}}$
    }
    
    \def\figwidth{0.5\linewidth}
    \centering
    \resizebox{\linewidth}{!}{
    \centering
    \begin{subfigure}[b]{\figwidth}
        \centering
        \includesvg[pretex=\fontsize{3.5}{8}\selectfont,width=\linewidth]{figures/adult-y1-eop.svg}
        \caption{Adult dataset: No.1 label group}
    \end{subfigure}
    \quad
    \begin{subfigure}[b]{\figwidth}
        \centering
        \includesvg[pretex=\fontsize{3.5}{8}\selectfont,width=\linewidth]{figures/adult-y18-eop.svg}
      \caption{Adult dataset: No.18 label group}
    \end{subfigure}
    }
    \\
    \resizebox{\linewidth}{!}{
    \centering
    \begin{subfigure}[b]{\figwidth}%
        \centering%
        \includesvg[pretex=\fontsize{3.5}{8}\selectfont,width=\linewidth]{figures/credit-y1-eop.svg}
        \caption{Credit dataset: No.1 label group}
    \end{subfigure}
    \quad
    \begin{subfigure}[b]{\figwidth}
        \centering
        \includesvg[pretex=\fontsize{3.5}{8}\selectfont,width=\linewidth]{figures/credit-y9-eop.svg}
        \caption{Credit dataset: No.9 label group}
    \end{subfigure}
    }
    \caption{
    EOp-accuracy tradeoffs. EOp regularizer was unstable and ineffective when the advantaged group is small,
    {\fairnotion}, on the other hand, preserved similar tradeoff trend as DP on both large and small label groups.
    }
 \label{fig:fairnotion-tradeoff-eop}
\end{figure*}

\begin{figure*}[htb]
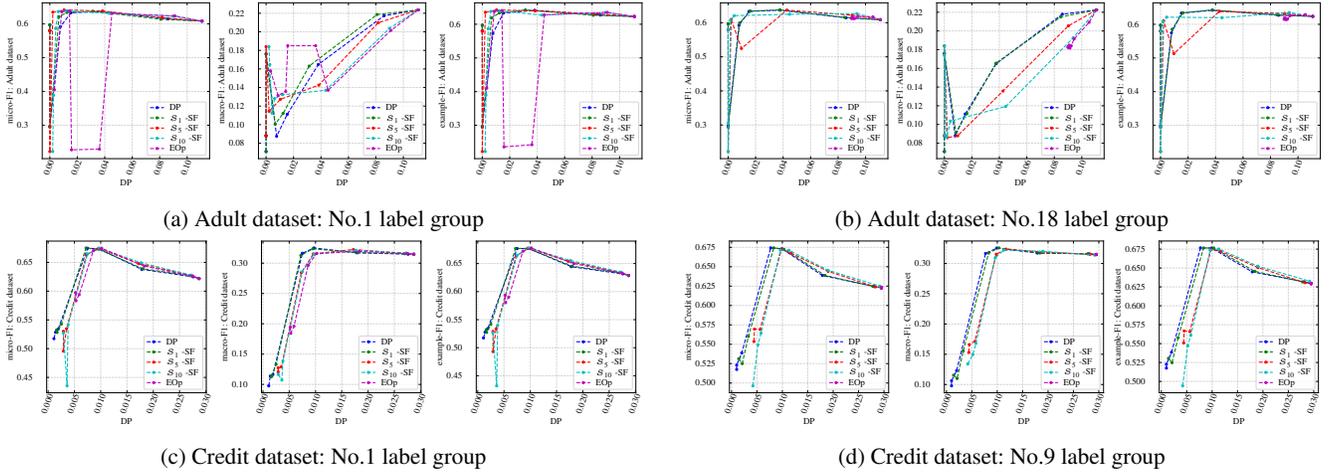

    \renewcommand{\t}[1]{
        $s_{\scalebox{0.8}{#1}}$
    }
    
    \def\figwidth{0.5\linewidth}
    \centering
    \resizebox{\linewidth}{!}{
    \centering
    \begin{subfigure}[b]{\figwidth}
        \centering
        \includesvg[pretex=\fontsize{3.5}{8}\selectfont,width=\linewidth]{figures/adult-y1-dp.svg}
        \caption{Adult dataset: No.1 label group}
    \end{subfigure}
    \quad
    \begin{subfigure}[b]{\figwidth}
        \centering
        \includesvg[pretex=\fontsize{3.5}{8}\selectfont,width=\linewidth]{figures/adult-y18-dp.svg}
      \caption{Adult dataset: No.18 label group}
    \end{subfigure}
    }
    \\
    \resizebox{\linewidth}{!}{
    \centering
    \begin{subfigure}[b]{\figwidth}%
        \centering%
        \includesvg[pretex=\fontsize{3.5}{8}\selectfont,width=\linewidth]{figures/credit-y1-dp.svg}
        \caption{Credit dataset: No.1 label group}
    \end{subfigure}
    \quad
    \begin{subfigure}[b]{\figwidth}
        \centering
        \includesvg[pretex=\fontsize{3.5}{8}\selectfont,width=\linewidth]{figures/credit-y9-dp.svg}
        \caption{Credit dataset: No.9 label group}
    \end{subfigure}
    }
    \caption{
    DP-accuracy tradeoffs. EOp regularizer was unstable and ineffective when the advantaged group is small,
    {\fairnotion}, on the other hand, preserved similar tradeoff trend as DP on both large and small label groups.
    }
 \label{fig:fairnotion-tradeoff-dp}
\end{figure*}

\subsection{Experiments on Multiclass Sensitive Feature}
\label{app:multi-class}

In this section we report results of {\fairnotion} on 
multiclass sensitive feature, where corresponding regularizers are constructed based on eqn \eqref{eq:fair-violation-k}. 
 We take Adult dataset as an example and use \textit{race} as the sensitive feature. All other experiment settings are same as before. 
 
Table \ref{tab:fairnotion-perform-multi-class} reports corresponding DP and EOp violations using different regularizers with coefficient $\lambda = 10$. Again, {\fairnotion} achieves the lowest DP and EOp violations as before. 
Figure \ref{fig:fairnotion-tradeoff-multi-class} shows corresponding fairness-accuracy tradeoffs, where {\fairnotion}
strikes a good tradeoff balances. 

\begin{table}[htb]
    \centering
    \begin{tabular}{rrr ccc ccc}
    \toprule
    & $|\yv_{adv}|$ & Metric & DP reg & $ s_{ 1 } $-SF reg & $ s_{5} $-SF reg & $ s_{10} $-SF reg & EOp reg & No reg\\
    \cmidrule(l){2-9}
    {\multirow{4}{*}{\rotatebox[origin=c]{90}{Adult}}} 
    & {\multirow{2}{*}{{No.1}}}   & DP  & 0.009 & 0.009 & \textbf{0.008} & \textbf{0.008}  & 0.051  & 0.110  \\
    &                             & EOp & 0.013 & 0.014 & 0.009 & \textbf{0.005}  & 0.041  & 0.160  \\
    \cmidrule(l){3-9}
    & {\multirow{2}{*}{{No.18}}}  & DP  & 0.009 & 0.010 & 0.004 & \textbf{0.002} & 0.105 & 0.110  \\
    &                             & EOp & 0.023 & 0.022 & 0.009 & \textbf{0.006} & 0.090 & 0.101  \\
    \bottomrule
    \end{tabular}
    \caption{
    DP and EOp violations of MPVAE trained with DP, EOp, and {\fairnotion} regularziers. 
    Sensitive feature is multiclass \textit{race}. 
    A large and a small advantaged groups (measured by their ranking in col. $|\yv_{adv}|$) are tested. 
    Results are averaged over 10 replications, best results are in bold. 
    }
    \label{tab:fairnotion-perform-multi-class}
\end{table}

\begin{figure*}[htb]
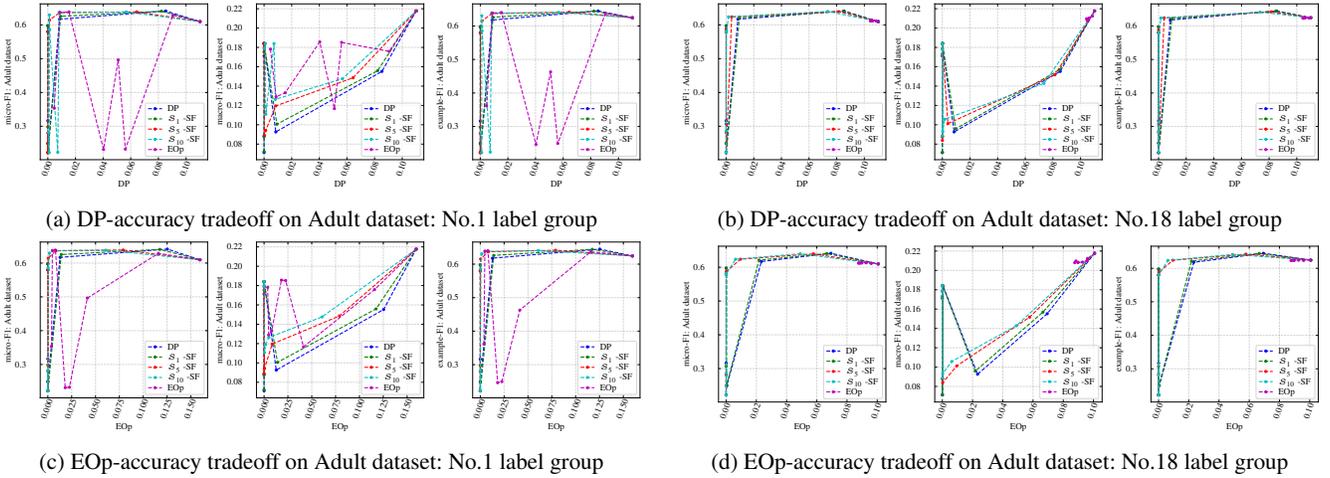

    \renewcommand{\t}[1]{
        $s_{\scalebox{0.8}{#1}}$
    }
    
    \def\figwidth{0.5\linewidth}
    \centering
    \resizebox{\linewidth}{!}{
    \centering
    \begin{subfigure}[b]{\figwidth}
        \centering
        \includesvg[pretex=\fontsize{3.5}{8}\selectfont,width=\linewidth]{figures/adult-y1-dp-race.svg}
        \caption{DP-accuracy tradeoff on Adult dataset: No.1 label group}
    \end{subfigure}
    \quad
    \begin{subfigure}[b]{\figwidth}
        \centering
        \includesvg[pretex=\fontsize{3.5}{8}\selectfont,width=\linewidth]{figures/adult-y18-dp-race.svg}
        \caption{DP-accuracy tradeoff on Adult dataset: No.18 label group}
    \end{subfigure}
    }
    \\
    \resizebox{\linewidth}{!}{
    \centering
    \begin{subfigure}[b]{\figwidth}%
        \centering%
        \includesvg[pretex=\fontsize{3.5}{8}\selectfont,width=\linewidth]{figures/adult-y1-eop-race.svg}
        \caption{EOp-accuracy tradeoff on Adult dataset: No.1 label group}
    \end{subfigure}
    \quad
    \begin{subfigure}[b]{\figwidth}
        \centering
        \includesvg[pretex=\fontsize{3.5}{8}\selectfont,width=\linewidth]{figures/adult-y18-eop-race.svg}
        \caption{EOp-accuracy tradeoff on Adult dataset: No.18 label group}
    \end{subfigure}
    }
    \caption{
    DP- and EOp-accuracy tradeoffs. 
    Sensitive feature is multiclass \textit{race}. 
    }
 \label{fig:fairnotion-tradeoff-multi-class}
\end{figure*}

\end{document}